\documentclass{article}

\usepackage{authblk}

\usepackage[top=30truemm,bottom=30truemm,left=25truemm,right=25truemm]{geometry}
\usepackage[square,numbers]{natbib}

\usepackage{hyperref}
\usepackage{url}

\usepackage{amsmath}
\usepackage{amssymb}
\usepackage{amsthm}
\usepackage{bm}
\usepackage{stmaryrd}

\usepackage{graphicx} % Required for inserting images
\usepackage[subrefformat=parens]{subcaption}

\usepackage{algorithm}
\usepackage{booktabs}

\usepackage{hyperref}
\usepackage{url}

\usepackage{multirow}

\usepackage[capitalize,nameinlink]{cleveref}
\crefname{section}{section}{sections}
\crefname{subsection}{subsection}{subsections}
\Crefname{section}{Section}{Sections}
\Crefname{subsection}{Subsection}{Subsections}

\Crefname{figure}{Figure}{Figures}

\crefformat{equation}{\textup{#2(#1)#3}}
\crefrangeformat{equation}{\textup{#3(#1)#4--#5(#2)#6}}
\crefmultiformat{equation}{\textup{#2(#1)#3}}{ and \textup{#2(#1)#3}}
{, \textup{#2(#1)#3}}{, and \textup{#2(#1)#3}}
\crefrangemultiformat{equation}{\textup{#3(#1)#4--#5(#2)#6}}%
{ and \textup{#3(#1)#4--#5(#2)#6}}{, \textup{#3(#1)#4--#5(#2)#6}}{, and \textup{#3(#1)#4--#5(#2)#6}}

\Crefformat{equation}{#2Equation~\textup{(#1)}#3}
\Crefrangeformat{equation}{Equations~\textup{#3(#1)#4--#5(#2)#6}}
\Crefmultiformat{equation}{Equations~\textup{#2(#1)#3}}{ and \textup{#2(#1)#3}}
{, \textup{#2(#1)#3}}{, and \textup{#2(#1)#3}}
\Crefrangemultiformat{equation}{Equations~\textup{#3(#1)#4--#5(#2)#6}}%
{ and \textup{#3(#1)#4--#5(#2)#6}}{, \textup{#3(#1)#4--#5(#2)#6}}{, and \textup{#3(#1)#4--#5(#2)#6}}

\crefdefaultlabelformat{#2\textup{#1}#3}

\newtheorem{theorem}{Theorem}[section]
\newtheorem{proposition}[theorem]{Proposition}

\newtheorem{definition}[theorem]{Definition}

\usepackage{algorithm}
\usepackage[noend]{algpseudocode}

\usepackage{cancel}
\allowdisplaybreaks

\usepackage{listings}

\usepackage{xcolor}

\definecolor{codegreen}{rgb}{0,0.6,0}
\definecolor{codegray}{rgb}{0.5,0.5,0.5}
\definecolor{codepurple}{rgb}{0.58,0,0.82}
\definecolor{backcolour}{rgb}{0.95,0.95,0.92}

\lstdefinestyle{mystyle}{
  backgroundcolor=\color{backcolour}, commentstyle=\color{codegreen},
  keywordstyle=\color{magenta},
  numberstyle=\tiny\color{codegray},
  stringstyle=\color{codepurple},
  basicstyle=\ttfamily\scriptsize,
  breakatwhitespace=false,
  breaklines=true,
  captionpos=b,
  keepspaces=true,
  numbers=left,
  numbersep=5pt,
  showspaces=false,
  showstringspaces=false,
  showtabs=false,
  tabsize=2
}
\DeclareMathOperator{\tr}{tr}

\title{Exploring Variance Reduction in Importance Sampling for Efficient DNN Training}
\author[$\ast$]{Takuro Kutsuna}
\affil[$\ast$]{\normalsize Toyota Central R\&D Labs., Inc.}
\date{}

\begin{document}
\maketitle
\noindent\textbf{Note.} This is the author’s accepted manuscript.
The final version is published in \textit{SIAM Journal on Mathematics of Data Science}
and is available via \texttt{https://doi.org/10.1137/25M1726339}.
\begin{abstract}
    Importance sampling is widely used to improve the efficiency of deep neural network (DNN) training by reducing the variance of gradient estimators.
    However, efficiently assessing the variance reduction relative to uniform sampling remains challenging due to computational overhead.
    This paper proposes a method for estimating variance reduction during DNN training using only minibatches sampled under importance sampling.
    By leveraging the proposed method, the paper also proposes an \textit{effective minibatch size} to enable automatic learning rate adjustment. An absolute metric to quantify the efficiency of importance sampling is also introduced as well as an algorithm for real-time estimation of importance scores based on moving gradient statistics. Theoretical analysis and experiments on benchmark datasets demonstrated that the proposed algorithm consistently reduces variance, improves training efficiency, and enhances model accuracy compared with current importance-sampling approaches while maintaining minimal computational overhead.
\end{abstract}

\section{Introduction} \label{sec:introduction}
Importance sampling is widely used in various research areas for two primary purposes.
First, it is used to estimate the expectation of a function~$f(x)$ of a random variable~$x$ under the target distribution~$p(x)$, where direct sampling from~$p(x)$ is challenging. Instead, an alternative distribution~$q(x)$, from which sampling is feasible, is used.
Applications include Bayesian posterior estimation \cite{neal2001annealed} and particle filters \cite{Liu2001}. Second, it aims to reduce the variance of the expectation estimator of~$f(x)$ by carefully selecting the alternative distribution~$q(x)$. Examples include rare event estimation \cite{JUNEJA2006291} and reliability analysis \cite{AU1999135}.

Importance sampling has also been applied to improve the efficiency of deep neural network (DNN) training \cite{alain2016variance,katharopoulos2018not,johnson2018training,katharopoulos2017biased,chang2017active}, which falls under the second purpose of variance reduction. In standard stochastic gradient descent (SGD) training, data samples are uniformly drawn from the training dataset to form minibatches.
In contrast, importance sampling-based approaches for DNN training estimate the importance score of each training sample using specific criteria during training, define an alternative sampling distribution proportional to these importance scores, and construct minibatches by sampling in accordance with this distribution.
To ensure unbiased estimation of the expected training loss, importance weights are applied to each sample.
Accurately estimating the importance score of each sample enables importance sampling to reduce the variance of the estimator for the expected loss (or its gradient) compared with uniform sampling in SGD.

The effectiveness of variance reduction achieved through importance sampling depends on the choice of the alternative distribution, i.e., the estimated importance score of each training sample in DNN-training applications; however, quantifying this effectiveness is challenging. While the variance-reduction rate can be evaluated by concurrently sampling from both the uniform distribution and alternative distribution to compute the variance of the loss for each case, this approach introduces significant computational overhead, making it impractical in most scenarios.
Although metrics, such as the effective sample size (ESS) \cite{kong1992note, liu1996metropolized} and its variants \cite{MARTINO2017386}, have been proposed as proxies for the effectiveness of importance sampling, they are insufficient for assessing variance-reduction efficiency in DNN training for the following reasons. While the primary objective of importance sampling in DNN training is to reduce the variance of the loss estimation compared with uniform sampling, it is implicitly assumed with ESS that importance sampling cannot achieve a variance smaller than that of uniform sampling \cite{MARTINO2017386}. However, this assumption does not necessarily hold in the context of DNN training with importance sampling.

To address this issue, this paper proposes a method for estimating the variance reduction achieved by importance sampling relative to uniform sampling during DNN training. It also estimates the lower bound of variance reduction achievable with the theoretically optimal alternative distribution. A key advantage of this method is that all estimations are executed using only minibatches sampled from the alternative distribution. Thus, the computational overhead remains minimal.

On the basis of the proposed method, this paper also proposes an effective minibatch size (EMS) for automatic learning-rate adjustment.
We derive the EMS~$N_\mathrm{ems}$ for importance sampling with a minibatch size of~$N$ such that uniform sampling with a minibatch size of~$N_\mathrm{ems}$ achieves the same variance as that obtained with importance sampling.
By leveraging the relationship between minibatch size and optimal learning rate~\cite{smith2018bayesian,smith2018don,li2024surge,mccandlish2018empirical,NEURIPS2024_ef74413c}, we derive an automatic learning-rate adjustment based on the EMS.
We also introduce an absolute metric to evaluate the efficiency of importance sampling and design an importance-score-estimation algorithm:
The metric assigns a value of~$0$ to indicate the theoretically optimal case and~$1$ to represent equivalence with uniform sampling. With this metric, the algorithm estimates the importance score of each data point during training by using the moving statistics of per-sample loss gradients. Specifically, the hyperparameter for the moving statistics, which determines the weight assigned to past observations, is designed on the basis of the proposed efficiency metric.

In summary, the contributions of this study are as follows:
\begin{itemize}
    \item We propose a method for estimating the variance reduction of the loss-gradient estimator in DNN training with importance sampling compared with uniform sampling, while maintaining minimal computational overhead (\cref{sec:variance_estimator}).
    \item We propose an EMS on the basis of the proposed method and apply it for automatic learning-rate adjustment (\cref{sec:ems}).
    \item We introduce a metric to evaluate the efficiency of importance sampling and designed an algorithm to estimate importance scores using moving statistics (\cref{sec:ema-is}).
    \item Experimental results on benchmark datasets indicate the superiority of the proposed method over current importance-sampling methods for DNN training (\cref{sec:experiment}).
\end{itemize}

\section{Related work} \label{sec:related_work}
\paragraph{Importance sampling for DNN training}
Several studies investigated the use of importance sampling to improve the efficiency of DNN training. Alain et al. \cite{alain2016variance} proposed a method for distributed training with importance sampling. With this method, the gradient norm of each training sample is computed across multiple computation nodes, aggregated on a master node, and used to carry out training with importance sampling. Although they discussed the variance reduction achieved through importance sampling, they did not explore efficient estimation methods or applications such as learning-rate adjustment, as proposed in this paper.
Johnson and Guestrin \cite{johnson2018training} investigated accelerating SGD training by analyzing the convergence speed of SGD to the optimal solution, leading to the development of importance sampling based on per-sample gradient norms. They proposed a robust regression model to estimate the importance score of each training sample. They also introduced a method for automatically adjusting the learning rate based on the estimated ``expected squared norm of the gradient.'' This method differs from the EMS-based learning-rate adjustment we derived.
Similarly, Katharopoulos and Fleuret \cite{katharopoulos2018not} proposed an importance-sampling method for accelerating the convergence speed of SGD to the optimal solution. They introduced a two-step approach for constructing minibatches with importance sampling. Importance scores are first calculated for a large initial minibatch (larger than the final minibatch size), then, a subsampling step is carried out on the basis of the computed importance scores to generate the final minibatch.
Katharopoulos and Fleuret \cite{katharopoulos2017biased} proposed using per-sample loss values as importance scores. To speed up the estimation of these loss values, they suggested training a smaller auxiliary model for loss estimation alongside the primary model.
Chang et al.~\cite{chang2017active} proposed a method for maintaining a history of model predictions during training and estimating per-sample importance scores using either the average or variance of these predictions.

Several studies explored the effectiveness of importance sampling in convex optimization problems \cite{pmlr-v37-zhaoa15,NIPS2014_f29c21d4,stich2017safe}. However, since DNN training involves non-convex problems, these methods are challenging to apply directly.

\paragraph{Other training approaches with non-uniform sampling}
Curriculum learning \cite{bengio2009curriculum,soviany2022curriculum} is another approach that uses non-uniform sampling of training data. The core idea is to reorder the training dataset on the basis of data difficulty, prioritizing easier samples during the early stages of training and gradually introducing harder samples as training progresses. This strategy aims to enhance the generalization performance of the final model.
In contrast, importance sampling executes training that is theoretically equivalent to uniform sampling but with reduced variance, as the estimated expected loss or gradient remains unbiased.

\paragraph{Complementary SGD improvement strategies}
Besides importance-sampling-based variance reduction, SGD has been improved through various approaches, ranging from optimizer design (e.g., Adam \cite{kingma2015adam}) to alternative variance-reduction techniques. The latter include control-variate-based methods such as SVRG \cite{johnson2013accelerating}, which reduce gradient variance by periodically computing full-batch gradients or maintaining historical gradient tables. While effective in convex or small-scale settings, these methods often face practical challenges in deep learning due to the breakdown of underlying assumptions \cite{defazio2019ineffectiveness}.
Our approach is orthogonal to both lines of research. As demonstrated in our experiment in \cref{sec:exp_adam}, it can be combined with the adaptive optimizer Adam by appropriately adjusting the learning rate, as discussed in \cref{sec:epsilon_ems}.

\paragraph{Effeciency measure in importance sampling}
ESS \cite{kong1992note,liu1996metropolized} was proposed as a metric to evaluate the efficiency of importance sampling and is widely used in applications such as detecting degeneracy in particle filters and sequential Monte Carlo methods \cite{Liu2001}. It is estimated on the basis of the ratio of the variance of an estimator~$\hat{\mu}$ under the target distribution~$\pi$, $\mathrm{Var}_\pi[\hat{\mu}]$, to the variance of an estimator~$\tilde{\mu}$ under the proposal distribution~$q$, $\mathrm{Var}_q[\tilde{\mu}]$.\footnote{The notations follow those in a previous study \cite{MARTINO2017386}.} This concept is closely related to the EMS we propose.
However, due to assumptions and approximations in its derivation, the property $\mathrm{Var}_\pi[\hat{\mu}] \leq \mathrm{Var}_q[\tilde{\mu}]$ is implicitly assumed with ESS \cite{MARTINO2017386}. The same limitation applies to its extension, generalized-ESS \cite{MARTINO2017386}. While this property does not pose significant issues for applications such as detecting degeneracy, it becomes a critical constraint when applying importance sampling to DNN training. This is because the primary objective of importance sampling in DNN training is to achieve lower variance of the estimator compared with uniform sampling, which serves as the target distribution.
In contrast, our EMS does not impose such constraints, making it suitable for applications of importance sampling in DNN training.

\section{Preliminary} \label{sec:preliminary}
In this section, we summarize key aspects of importance sampling for DNN training that are relevant to our research.

\subsection{Notations} \label{sec:notation}
Let the training dataset, consisting of~$M$ pairs of input~$x$ and ground truth label~$y$, be represented as~$\mathcal{D}_\mathrm{train} = \left\{\left(x^{(i)}, y^{(i)}\right)\right\}_{i=1}^M$. Each pair~$\left(x^{(i)}, y^{(i)}\right)$ is assumed to be independent and identically distributed (i.i.d.) in accordance with the distribution~$p(x, y)$.
Let~$f_\theta$ denote a DNN model to be trained, where~$\theta = \left(\theta_k\right)_{k=1}^K$ represents the model parameters. The training loss function for the $i$-th data point~$\left(x^{(i)}, y^{(i)}\right)$ is defined as~$\mathcal{L}(\theta; i) := \ell\left(f_\theta\left(x^{(i)}\right), y^{(i)}\right)$, where~$\ell$ is assumed to be a continuously differentiable function, such as the cross-entropy loss.
% The overall training loss is then given by~$\mathcal{L}(\theta) := \frac{1}{M} \sum_{i=1}^M \mathcal{L}(\theta;i)$.
%
The gradient of the loss function for the $i$-th data point with respect to~$\theta$ is denoted as~$\nabla_{\theta}\mathcal{L}(\theta; i)$. Since~$\theta$ represents a vector,~$\nabla_{\theta}\mathcal{L}(\theta; i)$ is also a vector. Specifically, the gradient with respect to the $k$-th parameter~$\theta_k$ is represented as~$\nabla_{\theta_k}\mathcal{L}(\theta; i) := \frac{\partial \mathcal{L}(\theta;i)}{\partial \theta_k}$, which is a scalar.
Additionally, for a natural number~$a$, let~$\llbracket a \rrbracket := \left\{1, \ldots, a\right\}$.
The expectation and variance (or covariance matrix) of a random variable~$z$ following the distribution~$p(z)$ are denoted as~$\mathbb{E}_{z \sim p(z)}\left[\cdot\right]$ and~$\mathbb{V}_{z \sim p(z)}\left[\cdot\right]$, respectively. Note that these become a vector and a matrix, respectively, when~$z$ is a vector.
Let~$\mathbb{R}^M_{>0}$ denote the $M$-dimensional space in which all elements are strictly positive.
Let~$\|\cdot\|$ denote the L2 norm of a vector.

\subsection{DNN training based on uniform sampling}
At each training step, we would like to know the true gradient~$G(\theta) := \mathbb{E}_{(x,y)\sim p(x,y)} \left[\nabla_\theta \ell\left(f_\theta\left(x\right), y\right)\right]$, which we don't have access to.
The goal of any gradient estimator is to approximate~$G(\theta)$ as well as possible with (a subset of) the training data.
In DNN training based on uniform sampling, the gradient of the loss is estimated by~$\nabla_{\theta}\mathcal{L}(\theta; i)$ with $i \sim p_\mathrm{unif}(i)$, where~$p_\mathrm{unif}(i)$ denotes a uniform distribution defined over~$\llbracket M \rrbracket$, the index set of~$\mathcal{D}_\mathrm{train}$, i.e.,~$p_\mathrm{unif}(i) = 1/M \ \left(\forall i \in \llbracket M \rrbracket\right)$.
We consider the following expectation:
\begin{align}
    \mathbb{E}_\mathrm{unif}\left[\nabla_\theta \mathcal{L}(\theta)\right] := \mathbb{E}_{i \sim p_\mathrm{unif}(i)}\left[\nabla_\theta\mathcal{L}(\theta; i)\right], \notag % \label{eq:unif_grad}
\end{align}
where~$\mathcal{L}(\theta)$ in the left hand side is used as a shorthand for the per-sample loss~$\mathcal{L}(\theta; i)$, and $\mathbb{E}_\mathrm{unif}\left[\nabla_\theta \mathcal{L}(\theta)\right]$ represents the expectation of the per-sample gradient under uniform sampling.
From the assumption of each~$(x^{(i)}, y^{(i)}) \in \mathcal{D}_\mathrm{train}$ being i.i.d. with respect to~$p(x,y)$, we have~$G(\theta) \approx \mathbb{E}_\mathrm{unif}\left[\nabla_\theta \mathcal{L}(\theta)\right]$.
In standard SGD training, we approximate~$\mathbb{E}_{\mathrm{unif}}\left[\nabla_\theta \mathcal{L}(\theta)\right]$ by computing the average gradient over a minibatch sampled according to~$p_\mathrm{unif}(i)$.\footnote{In SGD, sampling is often carried out
    without replacement (no resampling) within each epoch. However, for ease of
    comparison with importance sampling, we assume uniform sampling *with*
    replacement. For example, Chang et~al.\ \cite{chang2017active} refer to the
    former as SGD-Scan and the latter as SGD-Uni. Both methods were compared in
    our experiments.}
This minibatch gradient is an unbiased estimate of~$G(\theta)$.

\subsection{Training with importance sampling} \label{sec:is_training}
Let~$W \in \mathbb{R}^M_{>0}$ denote the vector of importance weights, where the $i$-th element~$w_i$ indicates the importance of $(x^{(i)}, y^{(i)}) \in \mathcal{D}_\mathrm{train}$.
In DNN training with importance sampling based on~$W$, the loss gradient is estimated as~$r(i;W)\,\nabla_{\theta}\mathcal{L}(\theta; i)$, where the index~$i$ is drawn from the distribution~$p_\mathrm{is}(i;W)$ defined over $\llbracket M \rrbracket$ by
\begin{align}
    p_\mathrm{is}(i;W) := \frac{w_i}{\sum_{i' \in \llbracket M \rrbracket} w_{i'}}.  \label{eq:prob_is}
\end{align}
This indicates~$p_\mathrm{is}(i;W) \propto w_i$.
The corresponding importance-sampling coefficient~$r(i;W)$ is given by
\begin{align}
    r(i;W) := \frac{p_\mathrm{unif}(i)}{p_\mathrm{is}(i;W)}.
\end{align}
We then define the weighted expectation as
\begin{align}
    \mathbb{E}_{\mathrm{is}(W)}\left[\nabla_\theta\mathcal{L}(\theta)\right] & := \mathbb{E}_{i \sim p_\mathrm{is}(i;W)}\left[r(i;W) \nabla_\theta\mathcal{L}(\theta; i)\right].  \label{eq:is_grad}
\end{align}
As shown in the next section,~$\mathbb{E}_\mathrm{unif}\left[\nabla_\theta\mathcal{L}(\theta)\right] = \mathbb{E}_{\mathrm{is}(W)}\left[\nabla_\theta\mathcal{L}(\theta)\right]$ holds.
Therefore, estimating~$\mathbb{E}_{\mathrm{is}(W)}\left[\nabla_\theta\mathcal{L}(\theta)\right]$ using a minibatch sampled from~$p_\mathrm{is}(i;W)$ provides an unbiased estimate of~$G(\theta)$.

% In DNN training using importance sampling with~$W$, the following expectation is computed:
% \begin{align}
%   \mathbb{E}_{\mathrm{is}(W)}\left[\nabla_\theta\mathcal{L}(\theta)\right] & := \mathbb{E}_{i \sim p_\mathrm{is}(i;W)}\left[r(i;W) \nabla_\theta\mathcal{L}(\theta; i)\right],  \label{eq:is_grad} \\
%   r(i;W)                                                                   & := \frac{p_\mathrm{unif}(i)}{p_\mathrm{is}(i;W)}, \notag
% \end{align}
% where~$p_\mathrm{is}(i;W)$ is a distribution defined over~$\llbracket M \rrbracket$ and given by~$p_\mathrm{is}(i;W) \propto w_i$, which indicates~$p_\mathrm{is}(i;W) = w_i / \sum_{i' \in \llbracket M \rrbracket} w_{i'}$.
% The term~$r(i;W)$ is referred to as the importance-sampling weight, or simply the weight, of data point~$i$.

\subsection{Expectation Equivalence}
The following relationship holds between the expected loss gradient under uniform sampling and that under importance sampling.
\begin{proposition} \label{prop:exp_equiv}
    For any~$W \in \mathbb{R}^M_{>0}$, it holds that
    \begin{align}
        \mathbb{E}_\mathrm{unif}\left[\nabla_\theta\mathcal{L}(\theta)\right] = \mathbb{E}_{\mathrm{is}(W)}\left[\nabla_\theta\mathcal{L}(\theta)\right]. \label{eq:grad_mean_eq}
    \end{align}
\end{proposition}
\begin{proof}
    See \cref{apdx:exp_equiv}.
\end{proof}
This proposition indicates that the expected loss gradient remains the same under both uniform sampling and importance sampling, regardless of~$W$.

\subsection{Comparison of variance} \label{sec:is_var}
The covariance matrices of the loss gradient for uniform sampling and importance sampling are defined as
\begin{align}
    \mathbb{V}_\mathrm{unif}\left[\nabla_\theta\mathcal{L}(\theta)\right]    & := \mathbb{V}_{i \sim p_\mathrm{unif}(i)}\left[\nabla_\theta\mathcal{L}(\theta; i)\right],   \notag      \\
    \mathbb{V}_{\mathrm{is}(W)}\left[\nabla_\theta\mathcal{L}(\theta)\right] & := \mathbb{V}_{i \sim p_\mathrm{is}(i;W)}\left[r(i;W)\nabla_\theta\mathcal{L}(\theta; i)\right].  \notag
\end{align}
Although the expectations of the loss gradient for uniform sampling and importance sampling are identical regardless of~$W$, as discussed above, their variances may vary depending on~$W$.

Following Alain et al. \cite{alain2016variance}, we consider the trace of the covariance matrices:
\begin{align}
    \tr\left(\mathbb{V}_\mathrm{unif}\left[\nabla_\theta \mathcal{L}(\theta)\right]\right)    & =\sum_{k=1}^K \mathbb{V}_\mathrm{unif}\left[\nabla_{\theta_k} \mathcal{L}(\theta)\right],  \notag    \\
    \tr\left(\mathbb{V}_{\mathrm{is}(W)}\left[\nabla_\theta \mathcal{L}(\theta)\right]\right) & =\sum_{k=1}^K \mathbb{V}_{\mathrm{is}(W)}\left[\nabla_{\theta_k} \mathcal{L}(\theta)\right].  \notag
\end{align}
This corresponds to summing the variances of the gradients for each~$\theta_k$, while ignoring the covariances between parameters.
The following result is known regarding the optimal importance sampling weight~\cite{alain2016variance}.
\begin{proposition}[Optimal importance sampling weight \cite{alain2016variance}]
    The importance sampling weight $W^\ast$ that minimizes the trace of the gradient variance $\tr\left(\mathbb{V}_{\mathrm{is}(W)}\left[\nabla_\theta \mathcal{L}(\theta)\right]\right)$ is given by
    \begin{align}
        W^\ast := \left\{w_i \ \middle| \ w_i = p_\mathrm{unif}(i) \left\|\nabla_\theta \mathcal{L}(\theta;i)\right\|, \ i\in\llbracket M \rrbracket \right\}, \label{eq:W_ast}
    \end{align}
    and the corresponding trace value at~$W^\ast$ is
    \begin{align}
        \tr\left(\mathbb{V}_{\mathrm{is}(W^\ast)}\left[\nabla_\theta \mathcal{L}(\theta)\right]\right) = \left(\mathbb{E}_{i \sim p_\mathrm{unif}(i)}\left[\left\|\nabla_\theta \mathcal{L}(\theta;i)\right\|\right]\right)^2 - \left\|\mathbb{E}_{i \sim p_\mathrm{unif}(i)}\left[\nabla_\theta \mathcal{L}(\theta;i)\right]\right\|^2. \label{eq:tr_V_is_ast}
    \end{align}
\end{proposition}
Considering that~$p_\mathrm{unif}(i) = 1/M$ and is independent of~$i$, it follows from \cref{eq:W_ast} that minimizing the trace of the covariance matrix of the loss gradient can be achieved by performing importance sampling with weights proportional to the L2 norm of the loss gradient~$\left\|\nabla_\theta \mathcal{L}(\theta; i)\right\|$ for each data point.
If the L2 norm of the loss gradient serves as a measure of the degree to which a data point~$i$ affects the training, prioritizing the sampling of data points with larger effects reduces the variance of the loss gradient.

\section{Variance estimation of loss gradient} \label{sec:variance_estimator}
In this section, we derive the formula for estimating the traces of the covariance matrices of the loss gradient under three settings: importance sampling with~$W$, uniform sampling, and importance sampling with the optimal~$W^\ast$. All these estimations can be efficiently carried out using only a sample drawn under importance sampling with~$W$.

\subsection{Trace formulas for gradient variance}
We extended the results by Alain et al. \cite{alain2016variance} to derive the following proposition.
\begin{proposition}[Trace of gradient variance under uniform and importance sampling] \label{prop:tr_cov_est}
    The traces of the covariance matrices of the loss gradient under importance sampling with~$W$, uniform sampling, and importance sampling with~$W^\ast$ are given as
    \begin{align}
        \tr\left(\mathbb{V}_{\mathrm{is}(W)}\left[\nabla_\theta \mathcal{L}(\theta)\right]\right)      & = \mathbb{E}_{i \sim p_\mathrm{is}(i;W)}\left[\left\|r(i;W) \nabla_\theta \mathcal{L}(\theta;i)\right\|^2\right] - \left\|\mu\right\|^2, \notag                    \\                   %\label{eq:trvar_is}                       \\
        \tr\left(\mathbb{V}_\mathrm{unif}\left[\nabla_\theta \mathcal{L}(\theta)\right]\right)         & = \mathbb{E}_{i \sim p_\mathrm{is}(i;W)}\left[r(i;W) \left\|\nabla_\theta \mathcal{L}(\theta;i)\right\|^2\right] - \left\|\mu\right\|^2,                    \notag \\ %\label{eq:trvar_unif} \\
        \tr\left(\mathbb{V}_{\mathrm{is}(W^\ast)}\left[\nabla_\theta \mathcal{L}(\theta)\right]\right) & = \left(\mathbb{E}_{i \sim p_\mathrm{is}(i;W)}\left[r(i;W) \left\|\nabla_\theta \mathcal{L}(\theta;i)\right\|\right]\right)^2 - \left\|\mu\right\|^2, \notag       %\label{eq:trvar_ideal_is}
    \end{align}
    where~$\mu$ is given by
    \begin{align}
        \mu & =  \mathbb{E}_{i \sim p_\mathrm{is}(i;W)}\left[r(i;W) \nabla_\theta \mathcal{L}(\theta;i)\right]. \notag %\label{eq:mu}
    \end{align}
\end{proposition}
\begin{proof}
    See \cref{apdx:proof_prop_tr_cov_est}.
\end{proof}

% \subsection{Pseudo code}
\subsection{Online variance estimation during minibatch-based training}
The key point of \cref{prop:tr_cov_est} is that all expectations in the formulas are taken with respect to~$p_\mathrm{is}(i;W)$. This allows the traces of the covariance matrices to be estimated using only samples drawn from~$p_\mathrm{is}(i;W)$.
Therefore, during training with minibatches generated via importance sampling, all relevant trace values can be estimated solely from the loss gradients of the data points within the minibatch.
For reference, \cref{apdx:pseudo_code} presents the pseudocode for variance estimation based on \cref{prop:tr_cov_est}, given~$r(i;W)$ and the loss gradients of each data point in a minibatch generated according to~$p_\mathrm{is}(i;W)$.

\section{Effective-minibatch-size estimation and learning-rate adjustment} \label{sec:ems}
% As discussed in previous sections, the variance of loss gradient estimation can be reduced by performing importance sampling with an appropriate~$W$.
% Another approach to reducing the variance of the loss gradient estimation is to increase the number of data points used for each estimation, which corresponds to the minibatch size in DNN training.
% In this section, with respect to variance reduction, we derive the relationship between minibatch-based training with a minibatch size of~$N$ under importance sampling and minibatch-based training with a minibatch size of~$N'$ under uniform sampling.
% We then proposed its application to automatic learning rate adjustment.
We derived an EMS that achieves the same variance reduction under uniform sampling as that achieved by importance sampling with importance weights~$W$ and a minibatch of size~$N$.
We then propose its application to automatic learning-rate adjustment.

\subsection{Effective minibatch size}
Our EMS is defined as follows.
\begin{definition}[Effective minibatch size]
    For DNN training with a minibatch size of~$N$ under importance sampling with~$W$, the~$N_\mathrm{ems}$ is defined as
    \begin{align}
        N_\mathrm{ems} := \frac{\tr\left(\mathbb{V}_\mathrm{unif}\left[\nabla_\theta\mathcal{L}(\theta)\right]\right)}{\tr\left(\mathbb{V}_{\mathrm{is}(W)}\left[\nabla_\theta\mathcal{L}(\theta)\right]\right)} N.  \label{eq:n_ems}
    \end{align}
\end{definition}
Note that the traces of the covariance matrices in \cref{eq:n_ems} can be efficiently estimated during DNN training by using \cref{prop:tr_cov_est}.
The following proposition holds for the EMS:
\begin{proposition} \label{prop:ems}
    The following two settings are equivalent in terms of the expectation and trace of the covariance matrix for the loss-gradient estimation:
    \begin{itemize}
        \item Training with a minibatch of size~$N$ under importance sampling with~$W$.
        \item Training with a minibatch of size~$N_\mathrm{ems}$ under uniform sampling.
    \end{itemize}
\end{proposition}
\begin{proof}
    See \cref{apdx:proof_prop_ems}.
\end{proof}
In contrast to ESS \cite{kong1992note, liu1996metropolized}, denoted as~$N_\mathrm{ess}$, which is implicitly restricted to~$N_\mathrm{ess} \leq N$ \cite{MARTINO2017386}, the proposed EMS can satisfy~$N_\mathrm{ems} > N$ depending on~$W$.
% Therefore, \cref{prop:ems} suggests that, with an appropriate~$W$, training under importance sampling with a minibatch of size~$N$ is equivalent to training under uniform sampling with a minibatch of size~$N_\mathrm{ems}$, where~$N_\mathrm{ems} > N$.
Therefore, \cref{prop:ems} suggests that, with an appropriate~$W$, training under importance sampling with a minibatch of size~$N$ is \emph{first-order equivalent}---in the sense that it matches both the expected stochastic gradient and the trace of its covariance matrix---to training under uniform sampling with a minibatch of size~$N_\mathrm{ems}$, where $N_\mathrm{ems}>N$.
This suggests that importance sampling with an appropriate~$W$ effectively increases the number of training epochs for a fixed number of training iterations, as training with a larger minibatch size corresponds to an increase in the number of effective epochs.

\subsection{Application to automatic learning rate adjustment} \label{sec:epsilon_ems}
It has been demonstrated that, in DNN training using SGD, increasing the minibatch size by a factor of~$\alpha$ and scaling the learning rate by~$\alpha^{-1}$ have equivalent effects on controlling the magnitude of noise during training~\cite{smith2018bayesian, smith2018don, li2024surge}. Therefore, under importance sampling with~$W$, the learning rate is implicitly scaled by a factor of~$N/N_\mathrm{ems}$, as the minibatch size is implicitly scaled by~$N_\mathrm{ems}/N$ according to \cref{prop:ems}.
This suggests that the effective learning rate in importance-sampling-based training may differ from the specified learning rate~$\epsilon$.
To address this issue, we propose adjusting~$\epsilon$ as follows to mitigate the impact of changes in the EMS caused by importance sampling:
\begin{align}
    \epsilon_\mathrm{ems} := \frac{N_\mathrm{ems}}{N}\epsilon = \frac{\tr\left(\mathbb{V}_\mathrm{unif}\left[\nabla_\theta\mathcal{L}(\theta)\right]\right)}{\tr\left(\mathbb{V}_{\mathrm{is}(W)}\left[\nabla_\theta\mathcal{L}(\theta)\right]\right)} \epsilon. \label{eq:epsilon_ems}
\end{align}
It should be noted that the optimal learning rate scaling rule when changing the minibatch size depends on the choice of optimizer. The above definition of~$\epsilon_\mathrm{ems}$ assumes the use of SGD as the optimizer. For the Adam optimizer, for example, it has been both empirically \cite{mccandlish2018empirical} and theoretically \cite{NEURIPS2024_ef74413c} suggested that, especially when the minibatch size is relatively small, the learning rate should be scaled by the square root of the minibatch size ratio. Therefore, when using the Adam optimizer, the learning rate can be adjusted based on~$N_\mathrm{ems}$ by setting $\epsilon_\mathrm{ems}^\mathrm{adam} := \sqrt{N_\mathrm{ems}/N}~\epsilon$ instead of using \cref{eq:epsilon_ems}.

\section{Designing importance-weight-estimation algorithm using variance estimators}  \label{sec:ema-is}
As discussed in \cref{sec:is_var}, setting~$W$ proportional to the per-sample gradient norm is optimal for variance reduction.
However, since the gradients evolve during training, directly computing their exact values at every step is computationally expensive and impractical.
We designed an algorithm to efficiently estimate the per-sample gradient norms using their moving statistics during training.
We first introduce an absolute metric to evaluate~$W$ called~$\mathcal{S}(W)$ that is based on \cref{prop:tr_cov_est}, enabling real-time monitoring of the effectiveness of importance sampling.
It is particularly useful for designing and assessing algorithms to estimate~$W$.
As an example, we use~$\mathcal{S}(W)$ to determine the hyperparameters for the moving statistics.

\subsection{Absolute metric for effectiveness of importance sampling} \label{sec:metric}
To evaluate the quality of~$W$, we introduce~$\mathcal{S}(W)$:
\begin{align}
    \mathcal{S}(W) := \frac{\tr\left(\mathbb{V}_{\mathrm{is}(W)}\left[\nabla_\theta \mathcal{L}(\theta)\right]\right) - \tr\left(\mathbb{V}_{\mathrm{is}(W^\ast)}\left[\nabla_\theta \mathcal{L}(\theta)\right]\right)}{\tr\left(\mathbb{V}_\mathrm{unif}\left[\nabla_\theta \mathcal{L}(\theta)\right]\right) - \tr\left(\mathbb{V}_{\mathrm{is}(W^\ast)}\left[\nabla_\theta \mathcal{L}(\theta)\right]\right)}. \notag
\end{align}
The quantities required to compute~$\mathcal{S}(W)$ can be efficiently evaluated using \cref{prop:tr_cov_est}.
The~$\mathcal{S}(W)$ takes values greater than or equal to zero and can be interpreted as follows.
\begin{itemize}
    \item When~$\mathcal{S}(W)$ is close to~$0$, the~$W$ perform nearly as well as~$W^\ast$.
    \item When~$\mathcal{S}(W) = 1$, importance sampling with~$W$ is equivalent to uniform sampling in terms of the variance of the loss-gradient estimation.
    \item When~$\mathcal{S}(W) > 1$, importance sampling with~$W$ results in a larger variance of the loss-gradient estimation compared with uniform sampling.
\end{itemize}
Therefore, if~$0 \leq \mathcal{S}(W) < 1$, importance sampling effectively reduces the variance, with smaller~$\mathcal{S}(W)$ indicating greater variance reduction.
Conversely, if~$\mathcal{S}(W) \geq 1$, importance sampling fails to reduce the variance compared with uniform sampling.

\subsection{Gradient-norm estimation using moving statistics} \label{sec:emais}
At training iteration~$t$, the gradient norm for data sample~$i$ is defined as~$g_i^t$:
\begin{align}
    g_i^t := \left\| \left.\nabla_\theta \mathcal{L}(\theta;i)\right|_{\theta = \theta^t} \right\|, \label{eq:per_sample_grad_norm}
\end{align}
where~$\theta^t$ represents the model parameters at~$t$.
Note that~$\theta^t$ is updated throughout the training process.

Our \cref{alg:ema_is} estimates~$W$ using moving statistics of~$g_i^t$ during training. In \cref{alg:ema_is}, the moving average of~$g_i^t$ is maintained as the internal state~$\hat{\mu}_i \ \left(i \in \llbracket M \rrbracket\right)$. At~$t$, when~$g_i^t$ is computed for~$i$, the moving average is updated via \textproc{UpdateStats}.
We assume that~$g_i^t$ is computed only for data samples included in the minibatch at~$t$, and \textproc{UpdateStats} is applied accordingly. Notably, when importance sampling is used, the intervals between iterations at which an~$i$ appears in the minibatch become non-uniform. To address this, \textproc{UpdateStats} extends the exponential moving average (EMA) to account for non-uniform time intervals \cite{eckner2012algorithms}.
Furthermore, \cref{alg:ema_is} estimates the moving variance~$\hat{\sigma}_i^2$ \cite{finch2009incremental} alongside the moving average~$\hat{\mu}_i$.
The~$W$ is then computed by summing these statistics in \textproc{ComputeImportance}.
The inclusion of~$\hat{\sigma}_i^2$ serves to prevent numerical instabilities, as discussed in \cite{alain2016variance,johnson2018training}.
To prevent the sampling probabilities derived from~$W$ from becoming excessively extreme, we do not directly use those computed by \cref{eq:prob_is}.
Instead, we employ the \textproc{ComputeAdjustedProbabilities} procedure in \cref{alg:ema_is} to obtain a flattened distribution.
This adjustment ensures that the expected maximum number of duplicates per minibatch does not exceed a specified hyperparameter~$\kappa$, which governs the degree of flattening.
In our experiments, we set~$\kappa = 1$.
\begin{algorithm}[tbhp]
    \caption{$W$ estimation with unevenly spaced moving statistics}
    \label{alg:ema_is}
    \textbf{[Hyper parameter]} $\tau (> 0)$: time constant for exponential decay in the moving statistics, $N$: minibatch size, $\kappa$: expected maximum duplicates per minibatch \\
    \textbf{[Inputs]} $t$: current iteration of training, $g_i^t$: gradient norm of the loss for data sample~$i$ at~$t$\\
    \textbf{[State variables]} $\hat{\mu}_i$: moving average for data sample~$i$, $\hat{\sigma}_i^2$: moving variance for~$i$, $t_i^\text{prev}$ : iteration where~$i$ was last evaluated
    \begin{algorithmic}[1]
        \Function{UpdateStats}{$t$, $g_i^t$}
        \State{$\alpha \leftarrow \exp \left(- \left(t - t_i^\text{prev}\right) / \tau \right)$} \label{alg:calc_alpha}
        \State{$\delta \leftarrow g_i^t - \hat{\mu}_i$}
        \State{$\hat{\mu}_i \leftarrow \hat{\mu}_i + \left(1 - \alpha\right) \delta$} \label{alg:update_mu}
        \State{$\hat{\sigma}_i^2 \leftarrow \alpha \left(\hat{\sigma}_i^2 + \left(1 - \alpha\right)\delta^2\right)$}
        \State{$t_i^\text{prev} \leftarrow t$}
        \EndFunction
        \Function{ComputeImportance}{ }
        \State{$w_i \leftarrow \hat{\mu}_i + \sqrt{\hat{\sigma}_i^2}, \quad \forall i \in \llbracket M \rrbracket$}
        \State{\Return{$W=\left(w_1, \ldots, w_M\right)^\top$}}
        \EndFunction
        \Function{ComputeAdjustedProbabilities}{$W$}
        \State{$p_i \leftarrow w_i / \sum_{i' \in \llbracket M \rrbracket} w_{i'}, \quad \forall i \in \llbracket M \rrbracket$}
        \While{$\max_i(p_i)\times N > \kappa$}
        \State{$p_i \leftarrow \sqrt{p_i} / \sum_{i' \in \llbracket M \rrbracket} \sqrt{p_{i'}}, \quad \forall i \in \llbracket M \rrbracket$}
        \EndWhile
        \State{\Return{$\left(p_1, \ldots, p_M\right)^\top$}}
        \EndFunction
    \end{algorithmic}
\end{algorithm}

\paragraph{Initialization of internal variables}
The internal state variables in \cref{alg:ema_is} are initialized as follows: We first conduct two epochs of training using uniform sampling without replacement, prior to applying importance sampling. The gradient norms~$g_i^t$ obtained during this uniform sampling phase are used to estimate the initial~$\hat{\mu}_i$ and~$\hat{\sigma}_i^2$.

\paragraph{Hyperparameter for moving statistics}
Hyperparameter~$\tau$ in \cref{alg:ema_is} determines the extent to which past observations affect the moving statistics.
For example, when~$\tau$ is small,~$\alpha$ computed in line~\ref{alg:calc_alpha} becomes small (closer to~$0$). Consequently, in line~\ref{alg:update_mu}, the moving average~$\hat{\mu}_i$ is updated to place greater emphasis on the current observation~$g_i^t$ rather than past values. Conversely, when~$\tau$ is large,~$\alpha$ approaches~$1$, causing the update to prioritize past values over the current observation.
In the next section, we examine the effect of~$\tau$ on~$W$ estimation using~$\mathcal{S}(W)$ and discuss strategies for determining an appropriate~$\tau$.

\subsection{Hyperparameter investigation using the proposed metric}
A straightforward approach to determining~$\tau$ is to vary it, conduct training, and selecting the~$\tau$ that achieves the highest prediction accuracy on a validation dataset. However, metrics, such as prediction accuracy, can be affected by factors other than~$\tau$, making it difficult to directly evaluate the quality of~$\tau$ or the resulting~$W$. To address this, we investigated the impact of~$\tau$ on the effectiveness of importance sampling using~$\mathcal{S}(W)$, which is estimated during training based on per-sample
gradients.

\subsubsection{Preliminary experiments using benchmark datasets} \label{sec:preliminary_exp}
We conducted importance-sampling-based training using~$W$, estimated with \cref{alg:ema_is}, on the CINIC-10 \cite{cinic10} and Fashion-MNIST (FMNIST) \cite{fmnist} datasets. Details of the training setup are provided in \cref{sec:train_setting}. The total number of training iterations was set to 70000 for CINIC-10 and 25000 for FMNIST. For each setting, we conducted training five times with different random seeds and evaluated the mean and standard deviation of~$\mathcal{S}(W)$ throughout the training process.

For the CINIC-10 dataset, we fixed~$\tau$ in \cref{alg:ema_is} to 5000, 10000, 30000, or 70000 and conducted training for each setting. The estimated~$\mathcal{S}(W)$ during training is shown in \cref{fig:is_scores:cinic10} (labeled as ``Fixed'' in the legend). The horizontal axis represents the training iterations, and the vertical axis represents~$\mathcal{S}(W)$.
For FMNIST, we fixed~$\tau$ to 1000, 5000, 10000, or 25000 and conducted training for each setting. The results are shown in \cref{fig:is_scores:fmnist}.
As described in \cref{sec:metric},~$\mathcal{S}(W)=0$ represents~$W^\ast$, while~$\mathcal{S}(W)=1$ indicates variance reduction equivalent to uniform sampling. Both are depicted in the figures as gray lines: a dashed line labeled ``Optimal IS'' and dash-dotted line labeled ``Uniform,'' respectively.
\begin{figure}[tbhp]
    \centering
    \subfloat[CINIC-10]{\label{fig:is_scores:cinic10}\includegraphics[width=0.5\linewidth,clip]{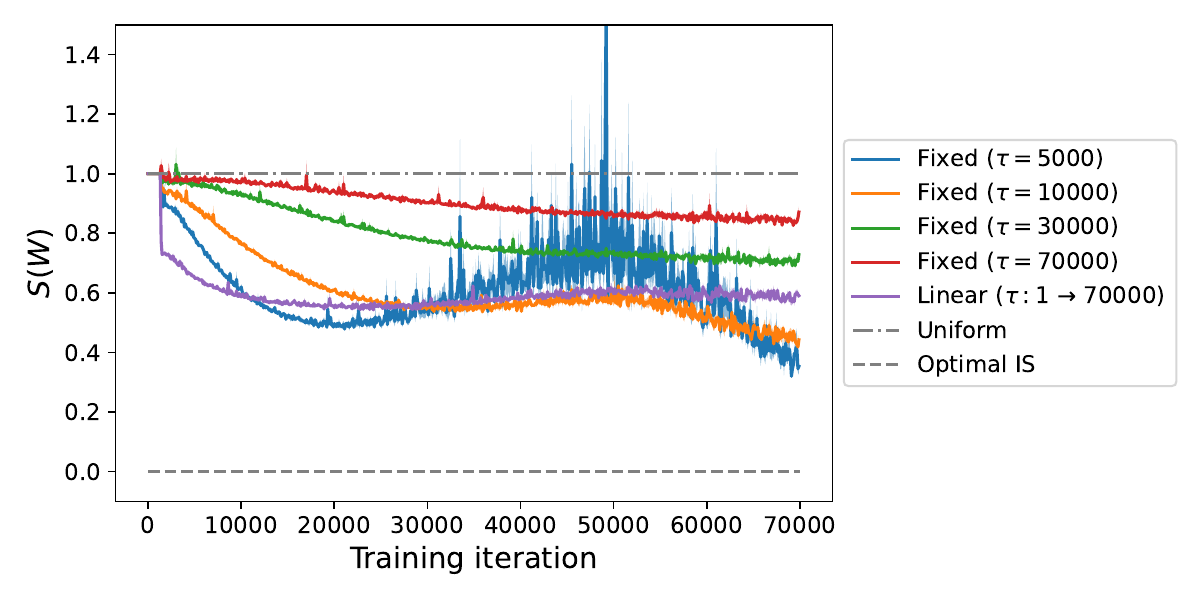}}
    \subfloat[FMNIST]{\label{fig:is_scores:fmnist}\includegraphics[width=0.5\linewidth,clip]{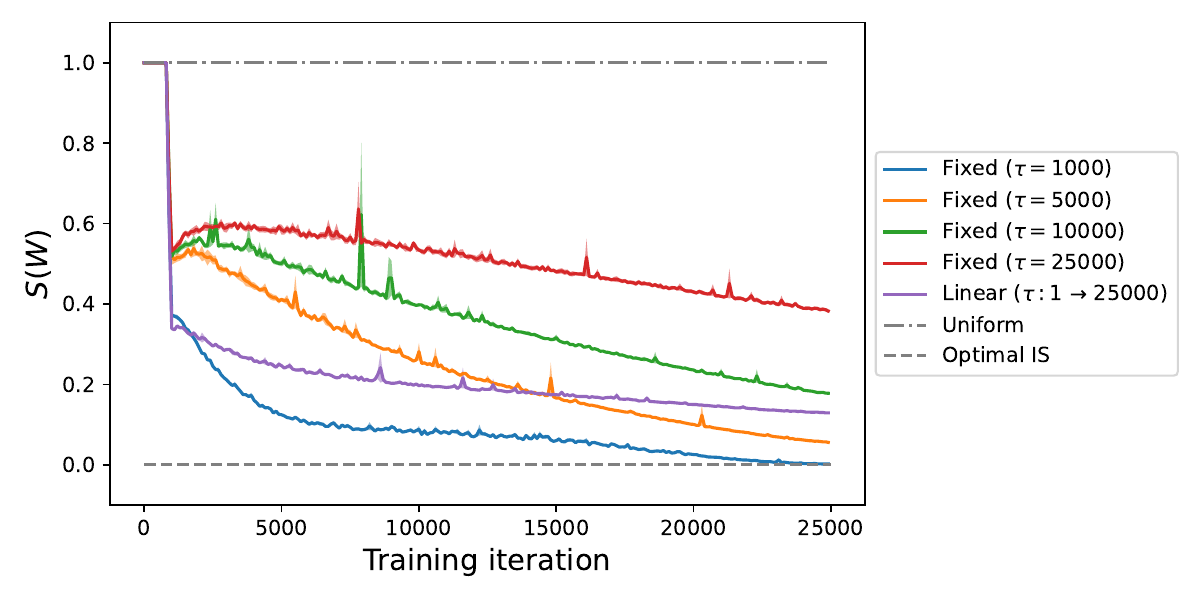}}
    \caption{Transitions of~$\mathcal{S}(W)$ during training for CINIC-10 and FMNIST.}
    \label{fig:is_scores}
\end{figure}

As shown in \cref{fig:is_scores:cinic10}, when~$\tau$ is set relatively small (e.g.,~$\tau=5000$) for CINIC-10,~$\mathcal{S}(W)$ remains small during the early to middle stages of training but becomes larger and unstable in the later stages. In contrast, as~$\tau$ increases,~$\mathcal{S}(W)$ exhibits more stable behavior throughout training. However, when~$\tau$ is excessively large,~$\mathcal{S}(W)$ decreases more slowly.
As shown in \cref{fig:is_scores:fmnist}, for FMNIST, similar to CINIC-10, smaller fixed values of~$\tau$ result in smaller~$\mathcal{S}(W)$ during the early stages of training. However, unlike CINIC-10, even for the smallest value~$\tau=1000$,~$\mathcal{S}(W)$ remains stable throughout the middle stages of training and continues to stay small until the final stages.
These results suggests that the appropriate~$\tau$ can vary significantly depending on the dataset and model.

\subsubsection{Dynamic~$\tau$}
On the basis of the above results, we propose varying~$\tau$ during training instead of using a fixed value. Specifically, we set~$\tau$ to the current~$t$ by inserting~$\tau \leftarrow t$ just before line~\ref{alg:calc_alpha} in \cref{alg:ema_is}. Hereafter, we refer to this method as \emph{Linear-$\tau$}, which notably eliminates the need to tune the hyperparameter~$\tau$ in \cref{alg:ema_is}.

The training results for CINIC-10 using Linear-$\tau$ are shown in \cref{fig:is_scores:cinic10} (labeled as ``Linear'' in the legend). Linear-$\tau$ consistently achieved small~$\mathcal{S}(W)$ throughout the training process.
These results suggest that in the early stages of training, importance sampling is more efficient when~$w_i$ adapts to the most recent~$g_i^t$. However, as training progresses, estimating~$w_i$ on the basis of a longer-term history of~$g_i^t$ leads to more stable importance sampling.
The training results for FMNIST using Linear-$\tau$ are shown in \cref{fig:is_scores:fmnist}. While~$\mathcal{S}(W)$ was slightly larger compared with when~$\tau$ was fixed at~1000, it still maintained relatively small~$\mathcal{S}(W)$ throughout the training process.
Due to~$\mathcal{S}(W)$, which efficiently assesses the effectiveness of importance sampling across the training process, a more refined and effective design of~$\tau$ was achieved.
As a reference, the efficiency score~$\mathcal{S}(W)$ obtained using Linear-$\tau$, along with the corresponding estimated trace of gradient variances used in its computation, is provided in \cref{apdx:score_and_trvars}.

We also conducted training using Linear-$\tau$ while varying the total number of training iterations and evaluated the corresponding~$\mathcal{S}(W)$. The results in \cref{apdx:analysis_linear_tau} show that Linear-$\tau$ achieves consistently small and stable~$\mathcal{S}(W)$ across different total numbers of training iterations.

\section{Efficient per-sample gradient approximation via logit gradients} \label{sec:approx_by_logit}
The formula in \cref{prop:tr_cov_est}, along with the derived quantities~$N_{\mathrm{ems}}$ and~$\mathcal{S}(W)$, requires the computation of the loss gradient~$\nabla_\theta \mathcal{L}(\theta; i)$ for each data point.
However, calculating the loss gradient with respect to~$\theta$ introduces significant computational overhead.
To address this issue, methods have been proposed to use the gradient with respect to the output of the final layer of the DNN (referred to as the logits) instead of the gradient with respect to~$\theta$ \cite{johnson2018training,katharopoulos2018not}.
Compared with computing the gradient for~$\theta$, the gradient with respect to the logits can be obtained at a significantly lower computational cost.

By defining the logits as~$z = f_\theta(x)$, the training loss function can be written as~$\ell(z, y)$.
We then introduce the auxiliary function~$\tilde{\mathcal{L}}(z; i) := \ell(z, y^{(i)})$.
% We denote~$\tilde{\mathcal{L}}(z; i) := \ell(z^{(i)}, y^{(i)})$ with~$z^{(i)}=f_\theta (x^{(i)})$.
Instead of computing the parameter gradient~$\nabla_\theta \mathcal{L}(\theta; i)$, we first evaluate the logit gradient~$\nabla_z \tilde{\mathcal{L}}(z; i) := \left.\nabla_z \ell(z, y^{(i)})\right|_{z = f_\theta(x^{(i)})}$, and substitute~$\nabla_z \tilde{\mathcal{L}}(z; i)$ for~$\nabla_\theta \mathcal{L}(\theta; i)$ in the formula in~\cref{prop:tr_cov_est}, resulting in the following expressions for the trace of the variance-covariance matrices of the gradient with respect to the logits:
\begin{align}
    \tr\left(\mathbb{V}_{\mathrm{is}(W)}\left[\nabla_z \tilde{\mathcal{L}}(z)\right]\right)      & = \mathbb{E}_{i \sim p_\mathrm{is}(i;W)}\left[\left\|r(i;W) \nabla_z \tilde{\mathcal{L}}(z;i)\right\|^2\right] - \left\|\mu_z\right\|^2, \label{eq:trvar_is_logit}                    \\
    \tr\left(\mathbb{V}_\mathrm{unif}\left[\nabla_z \tilde{\mathcal{L}}(z)\right]\right)         & = \mathbb{E}_{i \sim p_\mathrm{is}(i;W)}\left[r(i;W) \left\|\nabla_z \tilde{\mathcal{L}}(z;i)\right\|^2\right] - \left\|\mu_z\right\|^2, \label{eq:trvar_unif_logit}                  \\
    \tr\left(\mathbb{V}_{\mathrm{is}(W^\ast)}\left[\nabla_z \tilde{\mathcal{L}}(z)\right]\right) & = \left(\mathbb{E}_{i \sim p_\mathrm{is}(i;W)}\left[r(i;W) \left\|\nabla_z \tilde{\mathcal{L}}(z;i)\right\|\right]\right)^2 - \left\|\mu_z\right\|^2, \label{eq:trvar_ideal_is_logit}
\end{align}
where~$\mu_z$ is given by
\begin{align}
    \mu_z & =  \mathbb{E}_{i \sim p_\mathrm{is}(i;W)}\left[r(i;W) \nabla_z \tilde{\mathcal{L}}(z;i)\right]. \notag %\label{eq:mu}
\end{align}
Note that, for example, in an~$N_\mathrm{c}$-class classification problem, the dimension of~$z$ is typically~$N_\mathrm{c}$, and thus the dimension of~$\nabla_z \tilde{\mathcal{L}}(z; i)$ is also~$N_\mathrm{c}$.
In contrast, the dimension of~$\nabla_\theta \mathcal{L}(\theta; i)$ is~$K$, corresponding to the number of trainable parameters.
Since~$N_\mathrm{c} \ll K$ in most cases, the two gradients differ in size, and their norms can vary substantially.
However, as demonstrated by previous studies \cite{johnson2018training} and by our experiments in \cref{sec:assess_logit_grad_est}, the norms of~$\nabla_z \tilde{\mathcal{L}}(z; i)$ and~$\nabla_\theta \mathcal{L}(\theta; i)$ exhibit a strong proportional relationship across different values of~$i$.
A similar trend is observed, for example, between~$\tr\left(\mathbb{V}_{\mathrm{is}(W)}\left[\nabla_\theta \mathcal{L}(\theta)\right]\right)$ and~$\tr\left(\mathbb{V}_{\mathrm{is}(W)}\left[\nabla_z \tilde{\mathcal{L}}(z)\right]\right)$.
Although the magnitudes of these quantities may differ significantly, our experiments (see \cref{sec:assess_logit_grad_est}) empirically confirm a strong linear relationship between them.

While the estimated trace of the gradient variance based on the logits exhibits only a linear relationship with that computed from the parameter gradient, our experiments confirmed that the derived quantities~$N_{\mathrm{ems}}$ and~$S(W)$ can be estimated with a reasonable degree of absolute accuracy.
This is attributed to the fact that~$N_{\mathrm{ems}}$ and~$S(W)$ are defined based on the ratio of variance traces and the ratio of differences between variance traces, respectively.
The estimation errors of~$N_{\mathrm{ems}}$ and~$S(W)$ were also empirically evaluated in \cref{sec:assess_logit_grad_est}.

Throughout the experiments in \cref{sec:experiment}, as well as the preliminary experiments in \cref{sec:preliminary_exp}, variance and score estimations were performed using gradients with respect to the logits as an approximation, except in \cref{sec:assess_logit_grad_est}, which discusses the approximation accuracy of estimations based on the logit gradient.
Note that~$g_i^t$ in \cref{alg:ema_is} was also approximated using the logit gradients.
Since the relationship $p_\mathrm{is}(i;W) = p_\mathrm{is}(i; \gamma W)$ holds for any~$\gamma > 0$, using logit-gradient norms in \cref{alg:ema_is} provides a reasonable basis for estimating~$W$, assuming that the norm of the logit gradient is proportional to that of the parameter gradient.

\section{Overall training procedure of EMAIS} \label{sec:overall_procedure}
The complete procedure for EMAIS training, incorporating the logit-based gradient approximation, is presented in \cref{alg:train_emais}.
It is important to note that, in~\cref{alg:train_emais}, both~$\nabla_z \tilde{\mathcal{L}}(z; i_k)$ and~$\mathcal{L}(\theta; i_k)$ are required for each data index~$i_k$. However, in practice, these quantities can be computed simultaneously in a single forward pass rather than evaluated separately: we first compute~$z^{(i_k)} = f_\theta(x^{(i_k)})$, then obtain~$\nabla_z \tilde{\mathcal{L}}(z; i_k)$ as~$\left.\nabla_z \ell(z, y)\right|_{z = z^{(i_k)}}$ and~$\mathcal{L}(\theta; i_k)$ as~$\ell(z^{(i_k)}, y^{(i_k)})$.
\begin{algorithm}[t]
    \caption{EMAIS training with logit-based gradient approximation}
    \label{alg:train_emais}
    \begin{algorithmic}[1]
        \Require $\{(x^{(i)},y^{(i)})\}_{i=1}^M$: training data, $\left\{\epsilon_t\right\}_{t=1}^T$: learning rate, $N$: minibatch size, $T$: total number of iterations, $t'$: number of iterations corresponding to two training epochs
        \For{$t=1$ to $t'$}  \Comment{Initialize EMAIS internal statistics}
        \State Sample indices~$\{i_k\}_{k=1}^N\sim p_{\mathrm{unif}}(i)$ without replacement
        \For{each~$i_k$}
        % \State $g^{t}_{i_k} \leftarrow \nabla_z \tilde{\mathcal{L}}(z;\,i_k)$
        \State $\textproc{UpdateStats}\left(t, \left\|\nabla_z \tilde{\mathcal{L}}(z; i_k)\right\|\right)$ \Comment{See \cref{alg:ema_is}}
        \EndFor
        \State $\hat{g} \leftarrow \nabla_\theta \left(\frac{1}{N}\sum_{k=1}^N \mathcal{L}(\theta;i_k)\right)$
        \State $\theta \leftarrow \theta -\epsilon_t \hat{g}$
        \EndFor
        \For{$t=t'+1$ to $T$}
        \State $W \leftarrow \textproc{ComputeImportance}()$ \Comment{See \cref{alg:ema_is}}
        \State $\left\{\tilde{p}_i\right\}_{i=1}^M \leftarrow \textproc{ComputeAdjustedProbabilities}(W)$ \Comment{See \cref{alg:ema_is}}
        \State Sample indices~$\{i_k\}_{k=1}^N$ according to $\left\{\tilde{p}_i\right\}_{i=1}^M$ with replacement
        \For{each~$i_k$}
        \State $g^{t}_{i_k} \leftarrow \nabla_z \tilde{\mathcal{L}}(z;\,i_k)$
        \State $\textproc{UpdateStats}(t, \left\|g^{t}_{i_k}\right\|)$ \Comment{See \cref{alg:ema_is}}
        \EndFor
        % \State Estimate $\tr\left(\mathbb{V}_{\mathrm{is}(W)}\left[\nabla_z \tilde{\mathcal{L}}(z)\right]\right)$, $\tr\left(\mathbb{V}_\mathrm{unif}\left[\nabla_z \tilde{\mathcal{L}}(z)\right]\right)$, and $\tr\left(\mathbb{V}_{\mathrm{is}(W^\ast)}\left[\nabla_z \tilde{\mathcal{L}}(z)\right]\right)$ using \cref{eq:trvar_is_logit,eq:trvar_unif_logit,eq:trvar_ideal_is_logit}, with~$\left\{g^{t}_{i_k}\right\}_{k=1}^N$ and $r(i_k, W) = (1/M) / \tilde{p}_{i_k}$. Denote these estimates by $\phi_\mathrm{is}$, $\phi_\mathrm{unif}$, and $\phi_\mathrm{ideal}$, respectively. \Comment{See also \cref{prop:ems} and \cref{fig:pseudo_code}}
        \State Estimate $\phi_\mathrm{is} \approx \tr\left(\mathbb{V}_{\mathrm{is}(W)}\left[\nabla_z \tilde{\mathcal{L}}(z)\right]\right)$, $\phi_\mathrm{unif} \approx \tr\left(\mathbb{V}_\mathrm{unif}\left[\nabla_z \tilde{\mathcal{L}}(z)\right]\right)$, and $\phi_\mathrm{ideal} \approx \tr\left(\mathbb{V}_{\mathrm{is}(W^\ast)}\left[\nabla_z \tilde{\mathcal{L}}(z)\right]\right)$ using \cref{eq:trvar_is_logit,eq:trvar_unif_logit,eq:trvar_ideal_is_logit}, with~$\left\{g^{t}_{i_k}\right\}_{k=1}^N$ and $r(i_k, W) = (1/M) / \tilde{p}_{i_k}$. \Comment{See also \cref{prop:tr_cov_est} and \cref{fig:pseudo_code}}
        \State $N_{\mathrm{ems}}
            \leftarrow \frac{\phi_\mathrm{unif}}{\phi_\mathrm{is}} N$
        \State $\epsilon_{\mathrm{ems}} \leftarrow \frac{N_{\mathrm{ems}}}{N} \epsilon_t$
        \State (Optional) $\mathcal{S}(W) \leftarrow \frac{\phi_\mathrm{is} - \phi_\mathrm{ideal}}{\phi_\mathrm{unif} - \phi_\mathrm{ideal}}$
        \State $\hat{g} \leftarrow \nabla_\theta \left(\frac{1}{N}\sum_{k=1}^N r(i_k;W) \mathcal{L}(\theta; i_k)\right)$
        \State $\theta \leftarrow \theta - \epsilon_{\mathrm{ems}} \hat{g}$
        \EndFor
    \end{algorithmic}
\end{algorithm}

\section{Experiments} \label{sec:experiment}
In this section, we present the results of experiments comparing the proposed method with other methods on benchmark datasets.

\subsection{Datasets}
The following three datasets were used in the evaluation:
\emph{FMNIST} \cite{fmnist} consists of $28 \times 28$ grayscale images of clothing items labeled into ten classes. It contains 60000 images for training and 10000 images for testing.
\emph{CINIC-10} \cite{cinic10} is a dataset of $32 \times 32$ RGB images featuring objects such as animals and vehicles, labeled into ten classes. It contains 90000 images for training and 90000 images for testing.
The \emph{ChestX-ray14} dataset \cite{wang2017chestx} contains anonymized $1024 \times 1024$ grayscale chest X-ray images. Labels for 14 types of diseases (e.g., pneumonia) are assigned to each image by processing radiologist reports using natural language processing methods. Since a single X-ray image may be associated with multiple diseases described in the corresponding report, the dataset uses multi-label annotations. Specifically, each image is assigned 14 binary labels, where a value of~1 indicates the presence of a particular disease and~0 indicates its absence. Images with all 14 labels set to~0 correspond to ``No Finding'' (absence of any conditions).
It contains 86524 images for training and 25596 images for testing.

\subsection{Compared methods}
We used the following methods as comparisons, in which we refer to the proposed method as EMAIS (exponential moving average-based importance sampling):
\begin{itemize}
    \item \emph{SGD-Scan, SGD-Uni}: SGD-Scan generates training minibatches via uniform sampling without replacement, whereas SGD-Uni applies uniform sampling with replacement.
    \item \emph{RAIS} \cite{johnson2018training}: In RAIS,~$W$ is estimated using a robust regression model based on the per-sample gradient norm, and training is conducted using importance sampling with this~$W$. The learning rate is also adjusted in accordance with the expected squared norm of the gradient. RAIS also introduces the ``effective iteration number,'' a virtual iteration count used to adjust the learning rate scheduling accordingly.
    \item \emph{Presampling-IS} \cite{katharopoulos2018not}: This method generates training minibatches by first calculating per-sample gradient norms for a large initial minibatch of size~$N_\mathrm{large}$. A subsampling step is then executed using these scores to construct the final minibatch.
    \item \emph{Confidence} \cite{chang2017active}: With this method, the prediction probability~$p_{y_\mathrm{true}}^i$ for the correct class of training data point~$i$ is recorded during training. The average prediction probability, denoted as~$\bar{p}_{y_\mathrm{true}}^i$, is then used to compute each element of~$W$ as~$w_i = 1 - \bar{p}_{y_\mathrm{true}}^i$. Following this study \cite{chang2017active}, instead of directly using~$w_i$, we compute the average of~$w_i$ across all data points and add this value as an offset to each~$w_i$. A simple average is used to estimate the loss gradient instead of using \cref{eq:is_grad}.
    \item \emph{Confidence Variance (ConfVar)} \cite{chang2017active}: ConfVar is a method proposed by Chang et al. \cite{chang2017active} that determines~$w_i$ based on the variance of~$p_{y_\mathrm{true}}^i$ recorded during training. Specifically, ConfVar prioritizes sampling data points the prediction probabilities for the correct class of which exhibit significant fluctuation. As with the Confidence method, ConfVar incorporates an offset by adding the average of~$w_i$ to each~$w_i$, and uses a simple average to estimate the loss gradient.
    \item \emph{Self-paced learning} \cite{kumar2010self}: Self-paced learning, a variant of curriculum learning \cite{bengio2009curriculum}, automatically estimates the difficulty of each data point and selects training data accordingly. This method introduces a hyperparameter~$K_\mathrm{sp}$ such that, at each training iteration, only data points with loss values below~$1/K_\mathrm{sp}$ are used for training. As training progresses and more data points achieve smaller loss values, the proportion of data used for training increases. We use the method proposed by \cite{kumar2010self} to adjust~$K_\mathrm{sp}$. It is important to note that self-paced learning was originally designed for batch learning and cannot be directly applied to minibatch-based DNN training. To address this, we use a straightforward approach: within each minibatch, only data points with loss values below~$1/K_\mathrm{sp}$ are used for training.
    \item \emph{EMAIS (proposed)}: EMAIS uses \Cref{alg:ema_is} to estimate~$W$ with Linear-$\tau$. The learning rate is adjusted during training in accordance with \cref{eq:epsilon_ems}.
\end{itemize}
\paragraph{Details for implementation and settings of each method}
For RAIS, the robust regression model was reused from the authors' implementation \cite{RAIS_code}. All other methods were implemented from scratch. The hyperparameters for RAIS, such as the variables used in regression, were directly adopted from the authors' implementation.
For Presampling-IS, the large initial minibatch size was set to~$N_\mathrm{large} = 1024$. With Self-paced learning, the initial value of~$K_\mathrm{sp}$ was set such that 60\% of the training data were used at the beginning, and~$\mu = 1.2$ was used to increase~$K_\mathrm{sp}$ during training.
Note that the application of the Confidence and ConfVar methods to ChestX-ray14 is not straightforward, as it contains multi-labeled data. Therefore, these methods were excluded from evaluation on ChestX-ray14 in the experiments.

\subsection{DNN model, training, and evaluation settings}  \label{sec:train_setting}
The DNN models and training configurations for each dataset are summarized in \cref{tab:setup}.
Specifically, LeNet5~\cite{lecun1989backpropagation}, ResNet18~\cite{resnet}, and ResNet50~\cite{resnet} are employed for the FMNIST, CINIC-10, and ChestX-ray14 datasets, respectively.
For each dataset, multiple experiments were conducted by varying the total number of training iterations (``Total-iters'' in \cref{tab:setup}). For CINIC-10, two different weight-decay settings were evaluated.
The minibatch size was fixed at~128 as a common setting across all datasets.
For all methods, the learning rate was initialized with the values specified in \cref{tab:setup} and decayed using cosine scheduling, ensuring that it reaches zero at the specified total number of training iterations.
\begin{table}[tbhp]
    \tabcolsep = 7pt
    \caption{DNN model and training setup for each dataset}
    \label{tab:setup}
    \begin{center}
        \scalebox{0.8}{
            \begin{tabular}{l|lllllll}
                \toprule
                Dataset      & Model                                  & Loss & Total-iters              & LR   & WD        & DA        & SD       \\
                \midrule
                FMNIST       & LeNet5 \cite{lecun1989backpropagation} & CE   & 6250/12500/25000         & 0.01 & 1e-3      & --        & --       \\
                CINIC-10     & ResNet18 \cite{resnet}                 & CE   & 17500/35000/70000/140000 & 0.1  & 1e-4/5e-4 & Crop,Flip & --       \\
                ChestX-ray14 & ResNet50 \cite{resnet}                 & BCE  & 17500/35000/70000        & 0.1  & 1e-4      & Crop,Flip & 0.5      \\
                \bottomrule
                \multicolumn{8}{l}{\small{LR: initial learning rate (decayed using cosine scheduling), WD: weight decay, DA: data augmentation,}} \\
                \multicolumn{8}{l}{\small{SD: stochastic depth \cite{stoc_depth}, CE: cross entropy, BCE: binary cross entropy}}
            \end{tabular}
        }
    \end{center}
\end{table}

For each configuration in \cref{tab:setup}, training was conducted five times with different random seeds, and the average and standard deviation of the test accuracy were measured. For FMNIST and CINIC-10, prediction error (lower is better) was used as the evaluation metric. For the multi-label ChestX-ray14 dataset, prediction performance was assessed using mean average precision (mAP), where higher values indicate better performance.

\subsection{Evaluation results}

\subsubsection{Comparison of EMAIS and SGD-Uni with dynamic minibatch size}
To investigate the relationship between EMS, defined in \cref{sec:ems}, and learning with importance sampling, we conducted the following experiment. We first trained a model on FMNIST using EMAIS with a minibatch of size~$N=128$ for a total of 25000 iterations, recorded the value of~$N_\mathrm{ems}$ at each iteration (note that~$N_\mathrm{ems}$ is computed online during EMAIS training). We then trained SGD-Uni by dynamically adjusting the minibatch size at each iteration on the basis of the recorded~$N_\mathrm{ems}$ while adjusting the learning rate accordingly using the method described in \cref{sec:epsilon_ems}.

\Cref{fig:fmnist_sgd_unif_n_ems} presents the results, including the training loss and test error for each model, as well as the transitions of~$N_\mathrm{ems}$ recorded during EMAIS training.
For comparison, the training loss and test error of SGD-Uni and SGD-Scan with a fixed mini-batch size of~$N=128$ are also shown.
The figure shows that both the training loss and test error follow nearly the same trend for EMAIS with a fixed minibatch size of~$N=128$ and SGD-Uni with a dynamic minibatch size of~$N=N_\mathrm{ems}$. This suggests that training with importance sampling, given an appropriate~$W$, effectively mimics the effect of increasing the minibatch size. Moreover,~$N_\mathrm{ems}$ serves as an effective estimator of the extent to which the minibatch size is adjusted by importance sampling.
The figure also illustrates the superior performance of EMAIS compared with SGD-Uni and SGD-Scan under the condition that the minibatch size is fixed at~$N=128$.
As described in \cref{sec:emais}, the reason~$N_\mathrm{ems}$ remains flat in the early stages of training is that EMAIS estimates the initial values of the moving statistics during this period, without applying importance sampling.
\begin{figure}[tbhp]
    \centering
    \includegraphics[width=1.0\linewidth,clip]{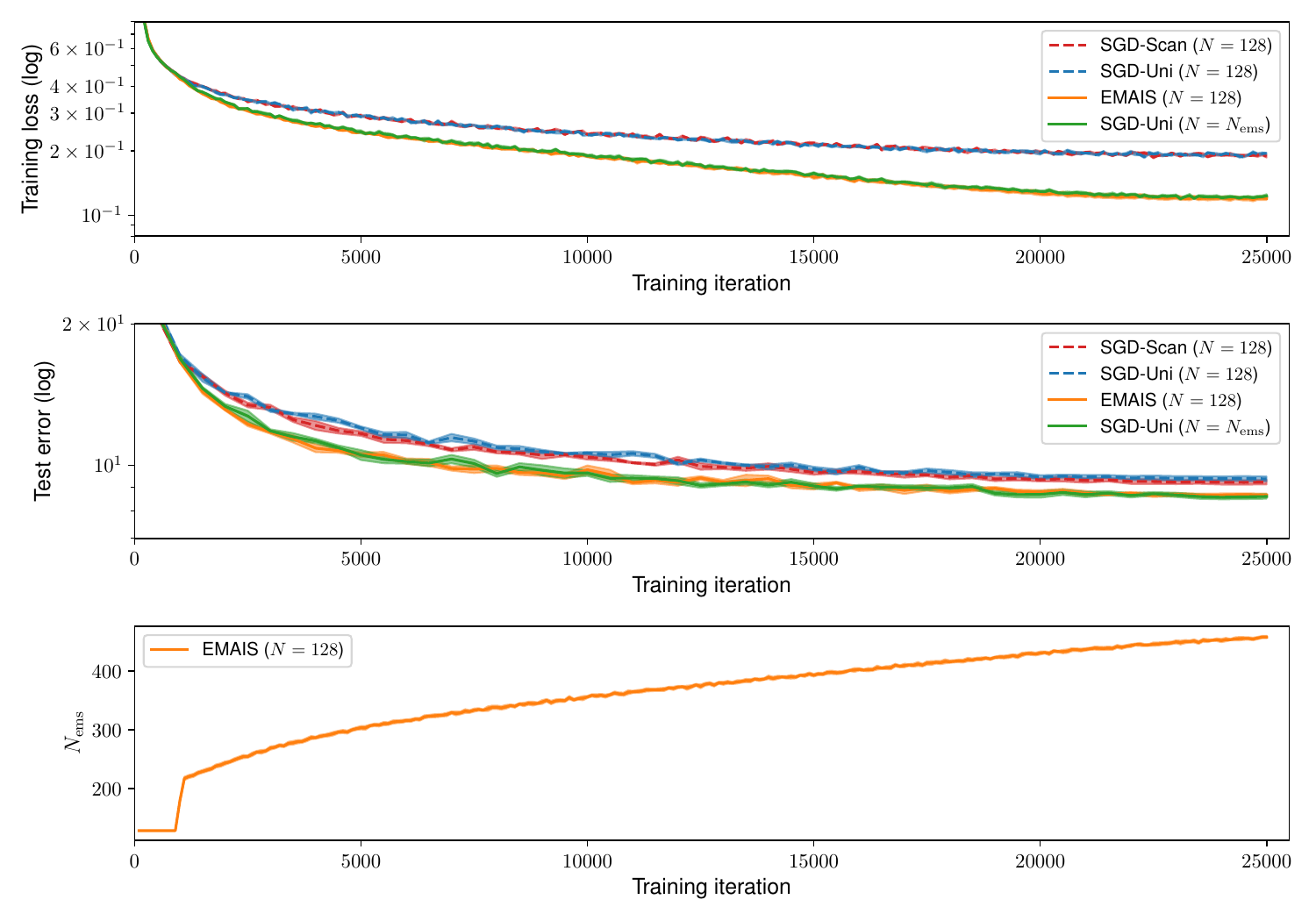}
    \caption{Training loss values (top), test error (middle), and EMS transitions (bottom) for FMNIST. Upper two plots compare three methods: SGD-Uni with fixed minibatch size of~$N=128$, EMAIS with a minibatch size of~$N=128$, and SGD-Uni with dynamic minibatch size~$N=N_\mathrm{ems}$, where~$N_\mathrm{ems}$ is reused from the values obtained during EMAIS training with~$N=128$ (shown in lower plot).}
    \label{fig:fmnist_sgd_unif_n_ems}
\end{figure}
% \begin{figure}[tbhp]
%   \centering
%   \includegraphics[width=0.9\linewidth,clip]{fmnist_sgd_unif_n_ems.pdf}
%   \caption{Training loss values (upper) and test error (lower) transitions for the FMNIST dataset. Three methods are compared: SGD-Uni with a fixed minibatch size of~$N=128$, EMAIS (proposed) with a minibatch size of~$N=128$, and SGD-Uni with a dynamic minibatch size~$N=N_\mathrm{ems}$, as defined in \cref{eq:n_ems}, where~$N_\mathrm{ems}$ is reused from values obtained during EMAIS training with~$N=128$.}
%   \label{fig:fmnist_sgd_unif_n_ems}
% \end{figure}

\subsubsection{Comparison of prediction accuracy}
The evaluation results of prediction accuracy for FMNIST, CINIC-10, and ChestX-ray14 are summarized in \cref{tab:res_fmnist}, \cref{tab:res_cinic10}, and \cref{tab:res_nih}, respectively. Each table shows the mean values of the evaluation metrics, with the numbers in parentheses indicating the standard deviations.
In each table, the best results for each setting are highlighted in bold. The results with mean values that fall within the standard deviation of the best result are also highlighted in bold.

From the tables, it is clear that EMAIS consistently achieved top performance across all datasets and settings, demonstrating its effectiveness. RAIS achieved the second-best accuracy and outperformed the standard SGD-Scan method with stable and high accuracy. These results suggest that importance sampling based on gradient norms is effective for DNN training.
In contrast, Presampling-IS exhibited instability in certain settings, particularly on CINIC-10. Confidence, ConfVar, and Self-paced Learning did not consistently yield better results than SGD-Scan. While the performance difference between SGD-Scan and SGD-Uni was not significant, SGD-Scan achieved slightly higher accuracy, especially on CINIC-10.

\begin{table}[tbhp]
    \centering
    % \footnotesize
    \tabcolsep = 5pt
    \caption{Prediction error on FMNIST test set ($\downarrow$)}
    \label{tab:res_fmnist}
    \scalebox{0.67}{
        \begin{tabular}{r|rrrrrrrr}
            \toprule
            Total-iters & \multicolumn{1}{c}{SGD-Scan} & \multicolumn{1}{c}{SGD-Uni} & \multicolumn{1}{c}{RAIS}   & \multicolumn{1}{c}{Presampling-IS} & \multicolumn{1}{c}{Confidence} & \multicolumn{1}{c}{ConfVar} & \multicolumn{1}{c}{Self-paced} & \multicolumn{1}{c}{EMAIS}   \\
            \midrule
            6250        & $12.02 (\pm 0.45)$           & $12.05 (\pm 0.07)$          & $10.90 (\pm 0.26)$         & $11.50 (\pm 0.13)$                 & $10.78 (\pm 0.31)$             & $12.30 (\pm 0.20)$          & $22.54 (\pm 3.34)$             & $\textbf{10.52} (\pm 0.10)$ \\
            12500       & $10.32 (\pm 0.11)$           & $10.30 (\pm 0.19)$          & $9.59 (\pm 0.26)$          & $10.22 (\pm 0.23)$                 & $9.35 (\pm 0.33)$              & $10.28 (\pm 0.16)$          & $13.24 (\pm 0.92)$             & $\textbf{9.05} (\pm 0.22)$  \\
            25000       & $9.21 (\pm 0.23)$            & $9.37 (\pm 0.25)$           & $\textbf{8.84} (\pm 0.18)$ & $\textbf{9.04} (\pm 0.18)$         & $9.11 (\pm 0.25)$              & $9.09 (\pm 0.25)$           & $9.72 (\pm 0.26)$              & $\textbf{8.83} (\pm 0.22)$  \\
            \bottomrule
        \end{tabular}
    }
\end{table}

\begin{table}[tbhp]
    \centering
    % \footnotesize
    \tabcolsep = 5pt
    \caption{Prediction error on CINIC-10 test set ($\downarrow$)}
    \label{tab:res_cinic10}
    \scalebox{0.64}{
        \begin{tabular}{lr|rrrrrrrr}
            \toprule
            WD   & Total-iters & \multicolumn{1}{c}{SGD-Scan} & \multicolumn{1}{c}{SGD-Uni} & \multicolumn{1}{c}{RAIS}    & \multicolumn{1}{c}{Presampling-IS} & \multicolumn{1}{c}{Confidence} & \multicolumn{1}{c}{ConfVar} & \multicolumn{1}{c}{Self-paced} & \multicolumn{1}{c}{EMAIS}   \\
            \midrule
            1e-4 & 17500       & $16.69 (\pm 0.09)$           & $16.89 (\pm 0.05)$          & $16.66 (\pm 0.13)$          & $16.76 (\pm 0.13)$                 & $17.08 (\pm 0.19)$             & $16.92 (\pm 0.21)$          & $17.86 (\pm 0.21)$             & $\textbf{16.42} (\pm 0.15)$ \\
                 & 35000       & $15.25 (\pm 0.15)$           & $15.52 (\pm 0.13)$          & $\textbf{14.98} (\pm 0.20)$ & $16.00 (\pm 1.75)$                 & $15.39 (\pm 0.09)$             & $15.57 (\pm 0.09)$          & $15.77 (\pm 0.10)$             & $\textbf{14.88} (\pm 0.11)$ \\
                 & 70000       & $14.14 (\pm 0.08)$           & $14.32 (\pm 0.15)$          & $14.02 (\pm 0.09)$          & $14.76 (\pm 0.50)$                 & $14.44 (\pm 0.07)$             & $14.32 (\pm 0.16)$          & $14.56 (\pm 0.13)$             & $\textbf{13.86} (\pm 0.09)$ \\
                 & 140000      & $13.49 (\pm 0.08)$           & $13.55 (\pm 0.11)$          & $13.31 (\pm 0.09)$          & $13.80 (\pm 0.10)$                 & $13.79 (\pm 0.13)$             & $13.52 (\pm 0.11)$          & $13.66 (\pm 0.09)$             & $\textbf{13.13} (\pm 0.03)$ \\
            % \midrule
            \addlinespace
            1e-5 & 17500       & $\textbf{15.27} (\pm 0.06)$  & $15.52 (\pm 0.13)$          & $\textbf{15.23} (\pm 0.03)$ & $\textbf{15.26} (\pm 0.11)$        & $15.61 (\pm 0.11)$             & $15.51 (\pm 0.07)$          & $16.18 (\pm 0.25)$             & $\textbf{15.19} (\pm 0.15)$ \\
                 & 35000       & $13.99 (\pm 0.11)$           & $14.21 (\pm 0.08)$          & $\textbf{13.83} (\pm 0.08)$ & $14.64 (\pm 1.07)$                 & $14.18 (\pm 0.13)$             & $14.05 (\pm 0.08)$          & $14.55 (\pm 0.16)$             & $\textbf{13.82} (\pm 0.14)$ \\
                 & 70000       & $13.15 (\pm 0.05)$           & $13.33 (\pm 0.07)$          & $13.06 (\pm 0.10)$          & $16.28 (\pm 2.85)$                 & $13.26 (\pm 0.06)$             & $13.36 (\pm 0.05)$          & $13.40 (\pm 0.05)$             & $\textbf{12.99} (\pm 0.06)$ \\
                 & 140000      & $\textbf{12.67} (\pm 0.09)$  & $12.86 (\pm 0.04)$          & $\textbf{12.75} (\pm 0.07)$ & $16.82 (\pm 3.19)$                 & $13.03 (\pm 0.06)$             & $12.89 (\pm 0.08)$          & $12.84 (\pm 0.07)$             & $\textbf{12.70} (\pm 0.07)$ \\
            \bottomrule
        \end{tabular}
    }
\end{table}

\begin{table}[tbhp]
    \centering
    % \footnotesize
    \tabcolsep = 5pt
    \caption{Mean average precision (mAP) on ChestX-ray14 test set ($\uparrow$)}
    \label{tab:res_nih}
    \scalebox{0.7}{
        \begin{tabular}{r|rrrrrr}
            \toprule
            Total-iters & \multicolumn{1}{c}{SGD-Scan} & \multicolumn{1}{c}{SGD-Uni} & \multicolumn{1}{c}{RAIS} & \multicolumn{1}{c}{Presampling-IS} & \multicolumn{1}{c}{Self-paced} & \multicolumn{1}{c}{EMAIS}   \\
            \midrule
            17500       & $18.66 (\pm 0.16)$           & $18.71 (\pm 0.03)$          & $20.30 (\pm 0.10)$       & $19.69 (\pm 0.14)$                 & $17.08 (\pm 0.71)$             & $\textbf{20.76} (\pm 0.10)$ \\
            35000       & $22.07 (\pm 0.10)$           & $21.92 (\pm 0.07)$          & $23.52 (\pm 0.17)$       & $23.43 (\pm 0.05)$                 & $21.35 (\pm 0.23)$             & $\textbf{23.72} (\pm 0.16)$ \\
            70000       & $24.69 (\pm 0.12)$           & $24.72 (\pm 0.14)$          & $25.39 (\pm 0.33)$       & $25.35 (\pm 0.04)$                 & $24.35 (\pm 0.23)$             & $\textbf{25.74} (\pm 0.07)$ \\
            \bottomrule
        \end{tabular}
    }
\end{table}

\subsubsection{Comparison of training time}
The training times for DNN models on FMNIST, CINIC-10, and ChestX-ray14 are summarized in~\cref{tab:per_iter_time}.\footnote{The experiments were conducted on a machine equipped with an $\text{Intel}^\text{\textregistered}$ $\text{Xeon}^\text{\textregistered}$ Gold 6134 CPU and an $\text{NVIDIA}^\text{\textregistered}$ $\text{Tesla}^\text{\textregistered}$ V100 GPU.}
The table reports the average training time per iteration for each dataset.\footnote{For detailed training times under each training configuration, please refer to~\cref{apdx:training_time}.}
The time required for DNN training can vary due to factors beyond the computational complexity of the learning algorithm, such as implementation details or the impact of other computational tasks when using shared computing resources. Therefore, while the current results are not suitable for a detailed comparison of computational complexity, they do provide an approximate understanding of the overall trends in training time.
\begin{table}[tbhp]
    \centering
    % \footnotesize
    \tabcolsep = 5pt
    \caption{Per-iteration average training time (in seconds)}
    \label{tab:per_iter_time}
    \scalebox{0.7}{
        \begin{tabular}{r|llllllll}
            \toprule
            Dataset      & \multicolumn{1}{c}{SGD-scan} & \multicolumn{1}{c}{SGD-unif} & \multicolumn{1}{c}{RAIS} & \multicolumn{1}{c}{Presampling-IS} & \multicolumn{1}{c}{Confidence} & \multicolumn{1}{c}{ConfVar} & \multicolumn{1}{c}{Self-paced} & \multicolumn{1}{c}{EMAIS} \\
            \midrule
            FMNIST       & $8.60 \times 10^{-3}$        & $8.62 \times 10^{-3}$        & $1.08 \times 10^{-2}$    & $5.22 \times 10^{-2}$              & $6.21 \times 10^{-3}$          & $9.13 \times 10^{-3}$       & $8.66 \times 10^{-3}$          & $7.32 \times 10^{-3}$     \\
            CINIC-10     & $6.62 \times 10^{-2}$        & $6.63 \times 10^{-2}$        & $6.63 \times 10^{-2}$    & $1.98 \times 10^{-1}$              & $6.35 \times 10^{-2}$          & $6.70 \times 10^{-2}$       & $6.50 \times 10^{-2}$          & $6.67 \times 10^{-2}$     \\
            ChestX-ray14 & $4.45 \times 10^{-1}$        & $4.45 \times 10^{-1}$        & $4.50 \times 10^{-1}$    & $1.36 \times 10^{0}$               & \multicolumn{1}{c}{--}         & \multicolumn{1}{c}{--}      & $4.45 \times 10^{-1}$          & $4.46 \times 10^{-1}$     \\
            \bottomrule
        \end{tabular}
    }
\end{table}

As shown in \cref{tab:per_iter_time}, the per-iteration training time differs across datasets, which is partly due to the size of the DNN models employed for each dataset.
When focusing on individual datasets, we observe no significant differences in training time among the methods, except for Presampling-IS.
For Presampling-IS, the computational overhead was notably large because it requires performing calculations on large minibatches beforehand, which resulted in the observed differences in training time.
Taking into account the results from the previous section and these findings, EMAIS can train more accurate models with approximately the same computational time as the widely used SGD-Scan.

\subsection{Evaluation with the Adam optimizer} \label{sec:exp_adam}
The above evaluation was conducted using SGD as the optimizer. In this section, we present the evaluation results obtained using the Adam optimizer. The experimental settings largely follow those in \cref{tab:setup}, except for the learning rate configuration, which was modified as follows (consistently applied across all datasets): the initial learning rate was set to 0.001 and decayed according to cosine scheduling.
We denote the method that samples minibatches using uniform sampling as \emph{Adam-Uni}, and the one that uses uniform sampling without resampling as \emph{Adam-Scan}. In addition to these, we evaluate RAIS and EMAIS using the Adam optimizer, referred to as \emph{RAIS-Adam} and \emph{EMAIS-Adam}, respectively.
For EMAIS-Adam, we applied the learning rate adjustment based on~$N_\mathrm{ems}$ for Adam, which is denoted as~$\epsilon_\mathrm{ems}^\mathrm{adam}$ in \cref{sec:epsilon_ems}.

The results are presented in \cref{tab:adam}, showing the mean values of the evaluation metrics, with the numbers in parentheses indicating the standard deviations over five trials. Bold values indicate either the best result or those whose mean falls within the standard deviation of the best result.
As shown in \cref{tab:adam}, EMAIS-Adam consistently outperforms the uniform sampling-based methods, Adam-Scan and Adam-Uni, across all datasets and training configurations.
These results demonstrate the effectiveness of importance sampling over uniform sampling, even when combined with the Adam optimizer.
When compared to RAIS-Adam, EMAIS-Adam shows comparable performance overall. In several configurations, such as those for CINIC-10, the two methods yield nearly identical results.
However, for ChestX-ray14, EMAIS consistently outperforms RAIS across all training durations.
This underscores the effectiveness of our method not only on standard classification tasks, but also on more complex problems such as medical image analysis with imbalanced multi-label data.
\begin{table}[tbhp]
    \centering
    % \footnotesize
    \tabcolsep = 5pt
    \caption{Evaluation results on FMNIST, CINIC-10, and ChestX-ray14 using the Adam optimizer. For FMNIST and CINIC-10, the values represent test error (\%, lower is better), while for ChestX-ray14, the values indicate mean average precision (mAP; higher is better).}
    \label{tab:adam}
    \scalebox{0.7}{
        \begin{tabular}{lllrrrr}
            \toprule
            Dataset                   & WD   & Total-iters & \multicolumn{1}{c}{Adam-Scan} & \multicolumn{1}{c}{Adam-Uni} & \multicolumn{1}{c}{RAIS-Adam} & \multicolumn{1}{c}{EMAIS-Adam} \\
            \midrule
            FMNIST ($\downarrow$)     & 1e-3 & 6250        & $10.99 (\pm 0.31)$            & $10.96 (\pm 0.14)$           & $10.20 (\pm 0.09)        $    & $\textbf{10.00} (\pm 0.13)$    \\
                                      &      & 12500       & $9.75 (\pm 0.10) $            & $9.81 (\pm 0.17) $           & $9.18 (\pm 0.22)         $    & $\textbf{8.92} (\pm 0.07) $    \\
                                      &      & 25000       & $9.09 (\pm 0.16) $            & $9.00 (\pm 0.25) $           & $\textbf{8.94} (\pm 0.21)$    & $\textbf{8.82} (\pm 0.17) $    \\
            \addlinespace
            CINIC-10 ($\downarrow$)   & 1e-4 & 17500       & $17.46 (\pm 0.08)$            & $17.90 (\pm 0.11)$           & $\textbf{17.32} (\pm 0.11)$   & $\textbf{17.43} (\pm 0.07)$    \\
                                      &      & 35000       & $17.02 (\pm 0.15)$            & $17.01 (\pm 0.11)$           & $\textbf{16.78} (\pm 0.13)$   & $\textbf{16.74} (\pm 0.09)$    \\
                                      &      & 70000       & $16.42 (\pm 0.14)$            & $16.56 (\pm 0.09)$           & $\textbf{16.12} (\pm 0.06)$   & $\textbf{16.13} (\pm 0.10)$    \\
                                      & 5e-5 & 17500       & $18.86 (\pm 0.23)$            & $18.87 (\pm 0.15)$           & $\textbf{18.36} (\pm 0.07)$   & $18.45 (\pm 0.16)         $    \\
                                      &      & 35000       & $17.73 (\pm 0.10)$            & $17.72 (\pm 0.16)$           & $\textbf{17.55} (\pm 0.13)$   & $\textbf{17.49} (\pm 0.15)$    \\
                                      &      & 70000       & $17.28 (\pm 0.22)$            & $17.40 (\pm 0.13)$           & $\textbf{17.18} (\pm 0.09)$   & $\textbf{17.23} (\pm 0.11)$    \\
            \addlinespace
            ChestX-ray14 ($\uparrow$) & 1e-4 & 17500       & $15.14 (\pm 0.21)$            & $15.19 (\pm 0.15)$           & $15.47 (\pm 0.24)$            & $\textbf{15.98} (\pm 0.10)$    \\
                                      &      & 35000       & $16.90 (\pm 0.13)$            & $17.03 (\pm 0.08)$           & $17.10 (\pm 0.24)$            & $\textbf{17.59} (\pm 0.14)$    \\
                                      &      & 70000       & $18.59 (\pm 0.19)$            & $18.46 (\pm 0.14)$           & $18.38 (\pm 0.23)$            & $\textbf{19.08} (\pm 0.14)$    \\
            \bottomrule
        \end{tabular}
    }
\end{table}

\subsection{Assessment of estimation with logit gradients} \label{sec:assess_logit_grad_est}
This section summarizes the results of evaluating the estimation accuracy of each metric when using gradients with respect to the logits, instead of the full parameter gradients discussed in \cref{sec:approx_by_logit}.
In this analysis, we used LeNet5 trained on FMNIST as a representative example, following the training configuration with 25000 total iterations as shown in~\cref{tab:setup}.
Although LeNet5 is a relatively small convolutional neural network, it contains over 60000 trainable parameters, making its parameter space substantially higher-dimensional than that of the logits, which have a dimensionality of~10 in this case.

We first present the results of comparing the per-sample gradient norms computed using full parameter gradients with those computed using logit gradients.
For models at training iterations 5000 and 25000, we computed both~$\left\|\nabla_\theta \mathcal{L}(\theta; i)\right\|$ and~$\left\|\nabla_z \tilde{\mathcal{L}}(z; i)\right\|$ for each of 500 data points~$i$ randomly sampled from the training dataset.
The results are presented in \cref{fig:full_logit_grad}.
Each scatter plot compares the logit-based gradient norm (horizontal axis) with the corresponding full parameter gradient norm (vertical axis) across data points.
As shown in the figure, although~$\left\|\nabla_\theta \mathcal{L}(\theta; i)\right\|$ and~$\left\|\nabla_z \tilde{\mathcal{L}}(z; i)\right\|$ differ in scale, they exhibit a clear proportional relationship. This observation is consistent with previous findings reported in~\cite{johnson2018training}.
\begin{figure}[tbhp]
    \centering
    \subfloat[For the model at iteration 5000]{\label{fig:full_logit_grad_5000}\includegraphics[width=0.35\linewidth,clip]{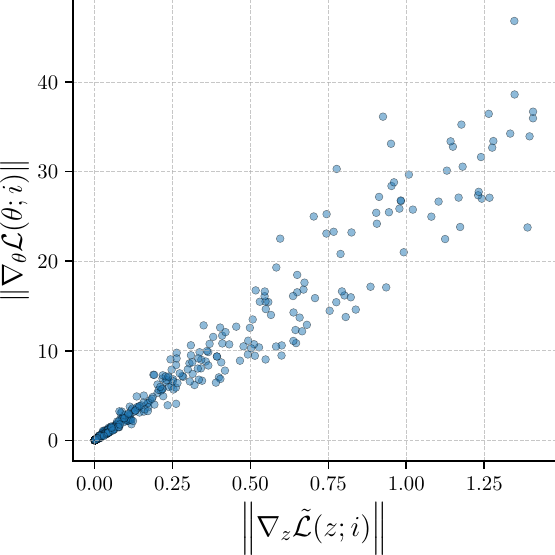}}
    \hspace{0.05\linewidth}
    \subfloat[For the model at iteration 25000]{\label{fig:full_logit_grad_25000}\includegraphics[width=0.35\linewidth,clip]{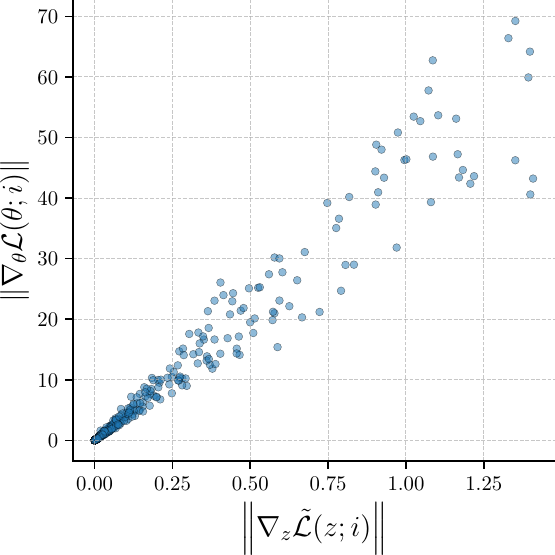}}
    \caption{Relationship between the full gradient and logit gradient norms for models at different training iterations.}
    \label{fig:full_logit_grad}
\end{figure}

Next, we evaluated the accuracy of trace estimates of the gradient variance based on the logit-gradient approximation, as discussed in~\cref{sec:approx_by_logit}, using models at training iterations 5000 and 25000.
For each model, we first computed the optimal importance weights~$W^\ast$, and then generated multiple variants of importance weight vectors~$W$ by adding controlled noise to~$W^\ast$. Using each~$W$, we constructed a single minibatch of size 128 via importance sampling.
For every data point~$i$ in the sampled minibatch, we computed both~$\nabla_\theta \mathcal{L}(\theta; i)$ and~$\nabla_z \tilde{\mathcal{L}}(z; i)$. These values were then used to estimate the trace of the gradient variance in accordance with \cref{prop:tr_cov_est} and \cref{eq:trvar_is_logit,eq:trvar_unif_logit,eq:trvar_ideal_is_logit}.
The results are presented in \cref{fig:full_logit_trvar}.
Each scatter plot compares the trace estimates of the gradient variance computed using logit gradients (horizontal axis) with those computed using full parameter gradients (vertical axis), across minibatches generated with different~$W$.
As shown in the figure, the trace estimates of the gradient variance computed using logit gradients exhibit a clear linear relationship with those computed using full parameter gradients, although the two quantities differ substantially in scale.
For reference, the fitted regression lines are shown in red in each plot.
While some variation is observed across models, both the estimated regression coefficients and intercepts remain within a similar order of magnitude.
\begin{figure}[htbp]
    \centering
    %--- First row: 5000 iterations ---%
    \subfloat[EMAIS (at 5000 iterations)\label{fig:tv_is5000}]{
        \includegraphics[width=0.3\linewidth]{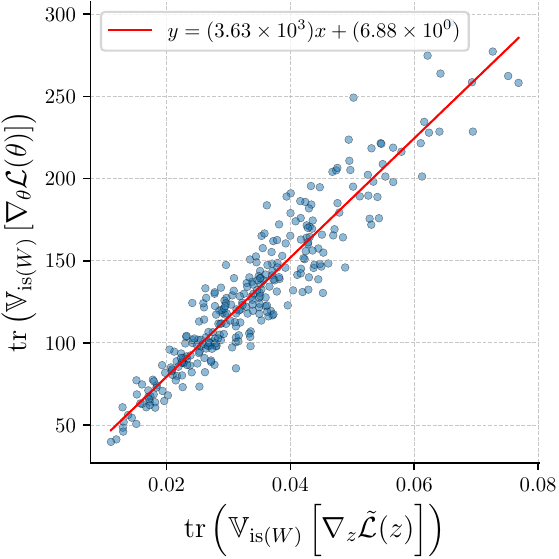}
    }\hfill
    \subfloat[Uniform (at 5000 iterations)\label{fig:tv_unif5000}]{
        \includegraphics[width=0.3\linewidth]{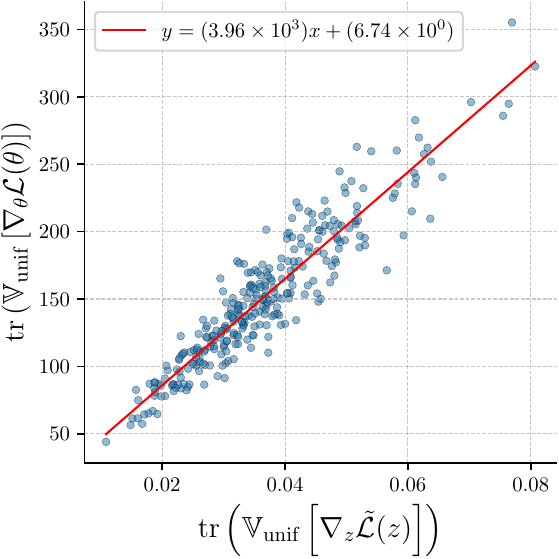}
    }\hfill
    \subfloat[Optimal IS (at 5000 iterations)\label{fig:tv_opt5000}]{
        \includegraphics[width=0.3\linewidth]{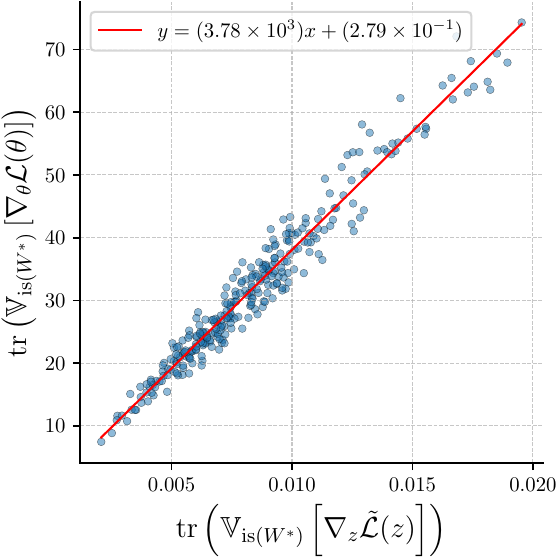}
    }\\
    %--- Second row: 25000 iterations ---%
    \subfloat[EMAIS (at 25000 iterations)\label{fig:tv_is25000}]{
        \includegraphics[width=0.3\linewidth]{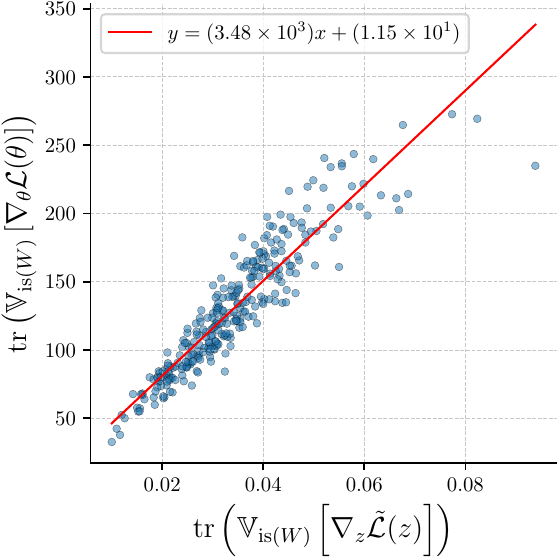}
    }\hfill
    \subfloat[Uniform (at 25000 iterations)\label{fig:tv_unif25000}]{
        \includegraphics[width=0.3\linewidth]{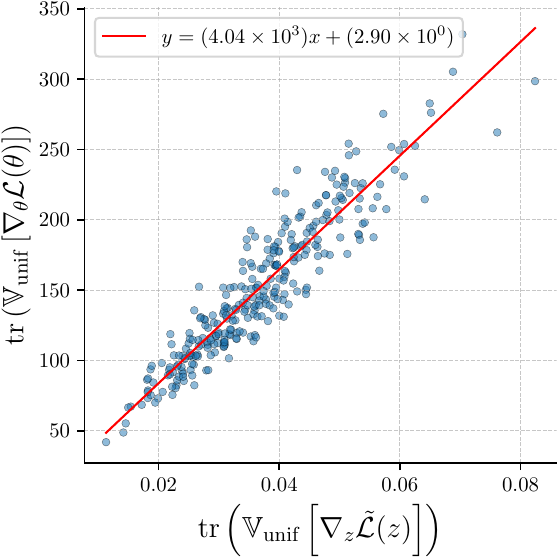}
    }\hfill
    \subfloat[Optimal IS (at 25000 iterations)\label{fig:tv_opt25000}]{
        \includegraphics[width=0.3\linewidth]{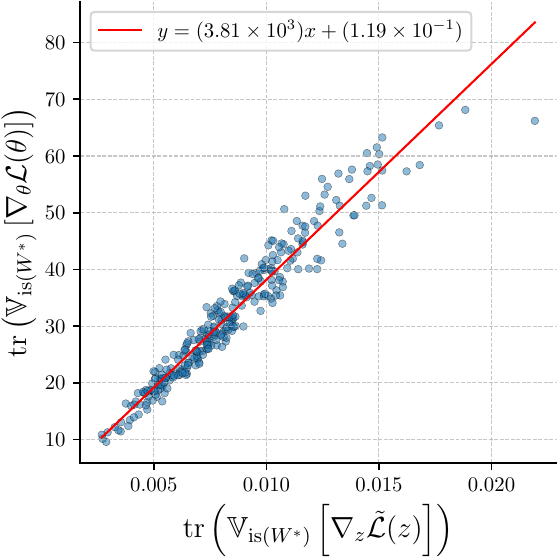}
    }
    \caption{%
        Comparison of trace variance estimates in \cref{prop:tr_cov_est}, computed using either full parameter gradients (plotted on the $y$-axis) or logit gradients (on the $x$-axis), as discussed in \cref{sec:approx_by_logit}, for models at training iterations 5000 and 25000. The red lines represent fitted regression lines.%
    }
    \label{fig:full_logit_trvar}
\end{figure}

Finally, we compared the estimates of~$N_{\mathrm{ems}}$ and~$\mathcal{S}(W)$ obtained using logit gradients and full parameter gradients, respectively.
For each minibatch generated by the procedure described above, we estimated the trace of the gradient variance and subsequently computed~$N_{\mathrm{ems}}$ and~$\mathcal{S}(W)$ based on these trace estimates.
The results are shown in~\cref{fig:full_logit_nems_score}.
Each scatter plot compares the estimated scores, either~$N_{\mathrm{ems}}$ or~$\mathcal{S}(W)$, computed using logit gradients (horizontal axis) with those computed using full parameter gradients (vertical axis), across the generated minibatches.
As shown in the figure, compared to estimates based on full parameter gradients, those obtained using logit gradients tend to slightly overestimate~$N_{\mathrm{ems}}$ and slightly underestimate~$\mathcal{S}(W)$.
Nevertheless, the two sets of estimates exhibit a strong linear relationship across all settings, indicating that the logit-based approximation can serve as a reliable substitute.
Notably, in contrast to gradient norms or trace variance estimates, the values of~$N_{\mathrm{ems}}$ and~$\mathcal{S}(W)$ computed using logit gradients remain within the same order of magnitude as those obtained from full gradients. This property is especially beneficial when absolute magnitudes influence downstream decisions, such as learning rate adjustment.
As discussed in~\cref{sec:approx_by_logit}, this consistency can be attributed to the fact that both~$N_{\mathrm{ems}}$ and~$\mathcal{S}(W)$ are computed based on the ratio (or the ratio of differences) of trace variance estimates, which inherently normalizes the scale differences between full and logit gradients.
\begin{figure}[htbp]
    \centering
    %--- First row: Effective minibatch size N_ems ---%
    \subfloat[$N_{\mathrm{ems}}$ (at 5000 iterations)\label{fig:ems5000}]{
        \includegraphics[width=0.35\linewidth]{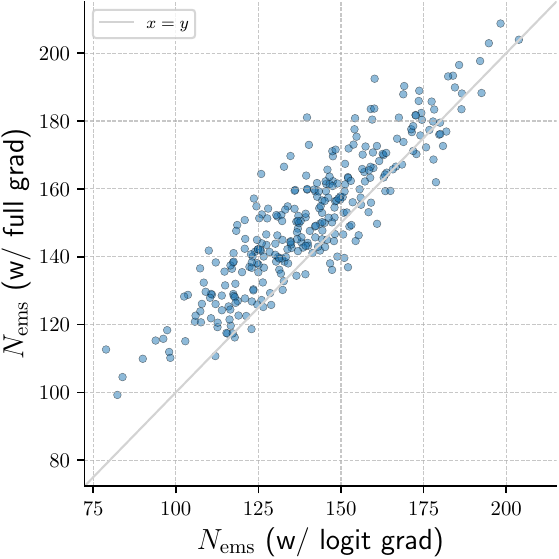}
    }
    \hspace{0.05\linewidth}
    \subfloat[$N_{\mathrm{ems}}$ (at 25000 iterations)\label{fig:ems25000}]{
        \includegraphics[width=0.35\linewidth]{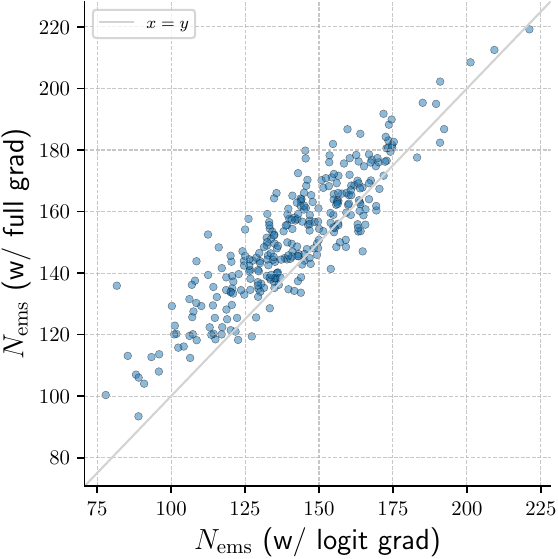}
    } \\
    %--- Second row: Efficiency metric S(W) ---%
    \subfloat[$\mathcal{S}(W)$ (at 5000 iterations)\label{fig:score5000}]{
        \includegraphics[width=0.35\linewidth]{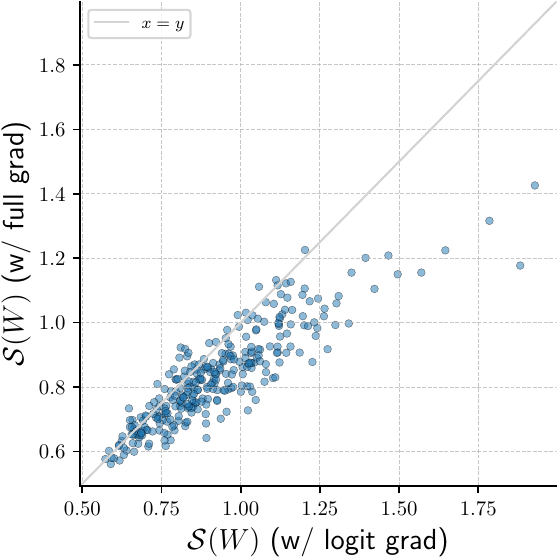}
    }
    \hspace{0.05\linewidth}
    \subfloat[$\mathcal{S}(W)$ (at 25000 iterations)\label{fig:score25000}]{
        \includegraphics[width=0.35\linewidth]{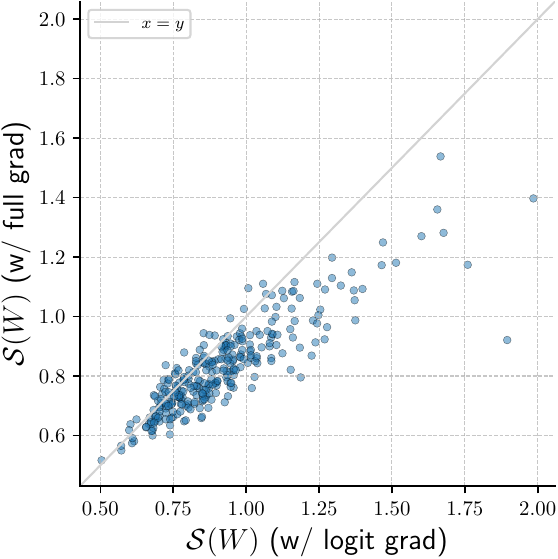}
    }
    \caption{%
        Scatter plots comparing estimates based on full parameter gradients versus logit-based approximations for (a, b) the effective minibatch size~$N_{\mathrm{ems}}$ and (c, d) the efficiency metric~$\mathcal{S}(W)$, for models at training iterations 5000 and 25000.%
    }
    \label{fig:full_logit_nems_score}
\end{figure}

\subsection{Ablation study of learning rate adjustment}
We report the results of an ablation study aimed at demonstrating the effectiveness of the learning rate adjustment proposed in \cref{sec:epsilon_ems}.
Using the FMNIST configuration with a total of~25000 iterations from \cref{tab:setup}, we trained the model with EMAIS, both with and without the learning rate adjustment.
% The results are shown in \cref{fig:fmnist_lr_adjust_ablation}, including the training loss, test error, effective minibatch size~$N_{\mathrm{ems}}$, and the learning rate~$\tilde{\epsilon}$ actually applied during training, which is given by~$\tilde{\epsilon}= \epsilon_\mathrm{ems}$ in \cref{eq:epsilon_ems} for EMAIS \emph{with} learning rate adjustment, and~$\tilde{\epsilon}= \epsilon$ for EMAIS \emph{without} learning rate adjustment.
% The figure also includes the scaled learning rate, which accounts for the virtual increase in minibatch size from~$N$ to~$N_\mathrm{ems}$ and is computed by~$\frac{N}{N_\mathrm{ems}} \tilde{\epsilon}$.
% For reference, the training loss and test error of SGD-Uni are also included.
The results are shown in \cref{fig:fmnist_lr_adjust_ablation}, including the training loss, test error, effective minibatch size~$N_{\mathrm{ems}}$, and the learning rate~$\tilde{\epsilon}$ actually applied during training.
For EMAIS \emph{with} learning rate adjustment, $\tilde{\epsilon} = \epsilon_{\mathrm{ems}}$ is computed as in \cref{eq:epsilon_ems}.
For EMAIS \emph{without} adjustment, $\tilde{\epsilon} = \epsilon$, which corresponds to the original learning rate with cosine decay.
The figure also includes the scaled learning rate, defined as $\frac{N}{N_{\mathrm{ems}}} \tilde{\epsilon}$, which accounts for the virtual increase in minibatch size from~$N$ to~$N_{\mathrm{ems}}$.
From the figure, we observe that, under the learning-rate-adjustment setting, the learning rate is scaled up with increasing~$N_{ems}$, such that the scaled learning rate follows the cosine decay schedule. By contrast, in the absence of learning-rate adjustment, although the actual learning rate used during training follows cosine decay, the scaled learning rate effectively applies a smaller rate than the decay schedule. Furthermore, this reduced scaled learning rate slows the decrease of the training loss and thereby degrades the test error reduction rate.
\begin{figure}[tbhp]
    \centering
    \includegraphics[width=0.9\linewidth,clip]{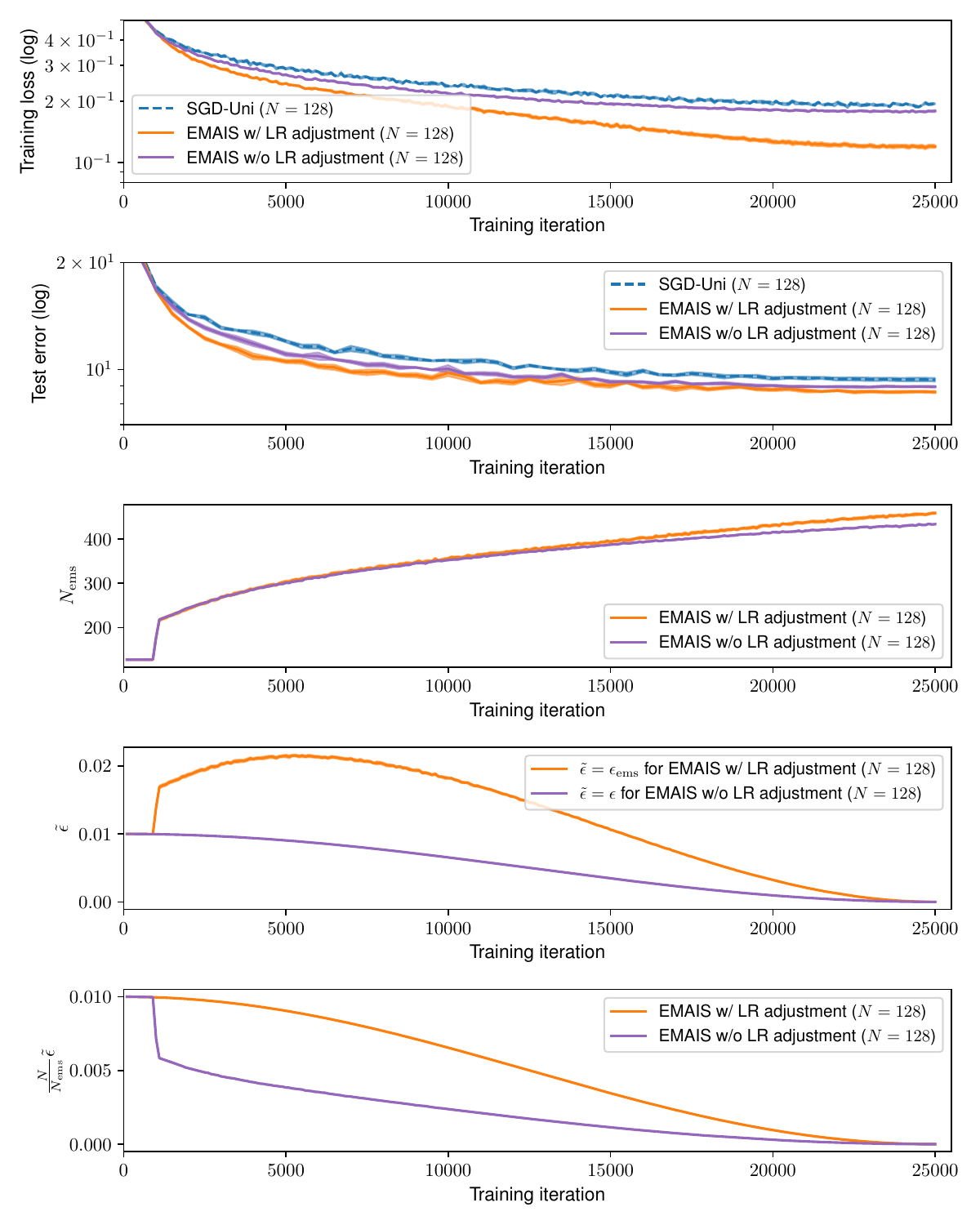}
    \caption{
        Ablation study of learning rate adjustment on FMNIST.
        From top to bottom, the panels show:
        (1) training loss (log scale);
        (2) test error (log scale);
        (3) effective minibatch size~$N_{\mathrm{ems}}$;
        (4) learning rate~$\tilde{\epsilon}$ actually applied during training, which is given by~$\tilde{\epsilon}= \epsilon_\mathrm{ems}$ for EMAIS with learning rate adjustment, and~$\tilde{\epsilon}= \epsilon$ for EMAIS without learning rate adjustment;
        and
        (5) scaled learning rate that accounts for the virtual increase in minibatch size from~$N$ to~$N_\mathrm{ems}$, given by~$\frac{N}{N_\mathrm{ems}} \tilde{\epsilon}$.
    }
    \label{fig:fmnist_lr_adjust_ablation}
\end{figure}

\subsection{Effect of additional updates to internal states in EMAIS}
In the proposed method, the importance sampling weights are computed based on a moving average of per-sample gradient norms, which are maintained as internal states. However, these internal states are updated only for the samples included in the mini-batches during training. In this study, we experimentally investigate whether training can be accelerated by periodically computing the gradient norms of samples that are not included in the current mini-batch and updating their corresponding importance weights.
As representative examples, we used the FMNIST and CINIC-10 configurations with total-iters~$25000$ and~$70000$, respectively, as listed in \cref{tab:setup}.

In addition to the internal state updates based on mini-batches, we performed supplementary updates to the internal states using the following procedure. For FMNIST, every~$500$ training iterations, and for CINIC-10, every~$1000$ iterations, we froze the model parameters and computed the per-sample gradient norms for the entire training set. These values were then used to invoke the \textproc{UpdateStats} procedure from \cref{alg:ema_is} to update the internal states. Note that these gradient computations were not used to update the model parameters.
A comparison of the training loss, test error, and the efficiency score~$\mathcal{S}(W)$ is presented in \cref{fig:additional_state_updates}.
As shown in the figure, the efficiency score~$\mathcal{S}(W)$ tends to be slightly lower when additional state updates are applied, particularly during the early stages of training, compared to the original EMAIS. However, the training loss and test error remain largely unchanged.
These findings suggest that, in the proposed algorithm, estimating importance sampling weights solely based on mini-batches during training---which entails lower computational overhead---has only a limited effect on overall training performance.
\begin{figure}[tbhp]
    \centering
    \subfloat[FMNIST]{\label{fig:additional_state_updates:fmnist}\includegraphics[width=0.5\linewidth,clip]{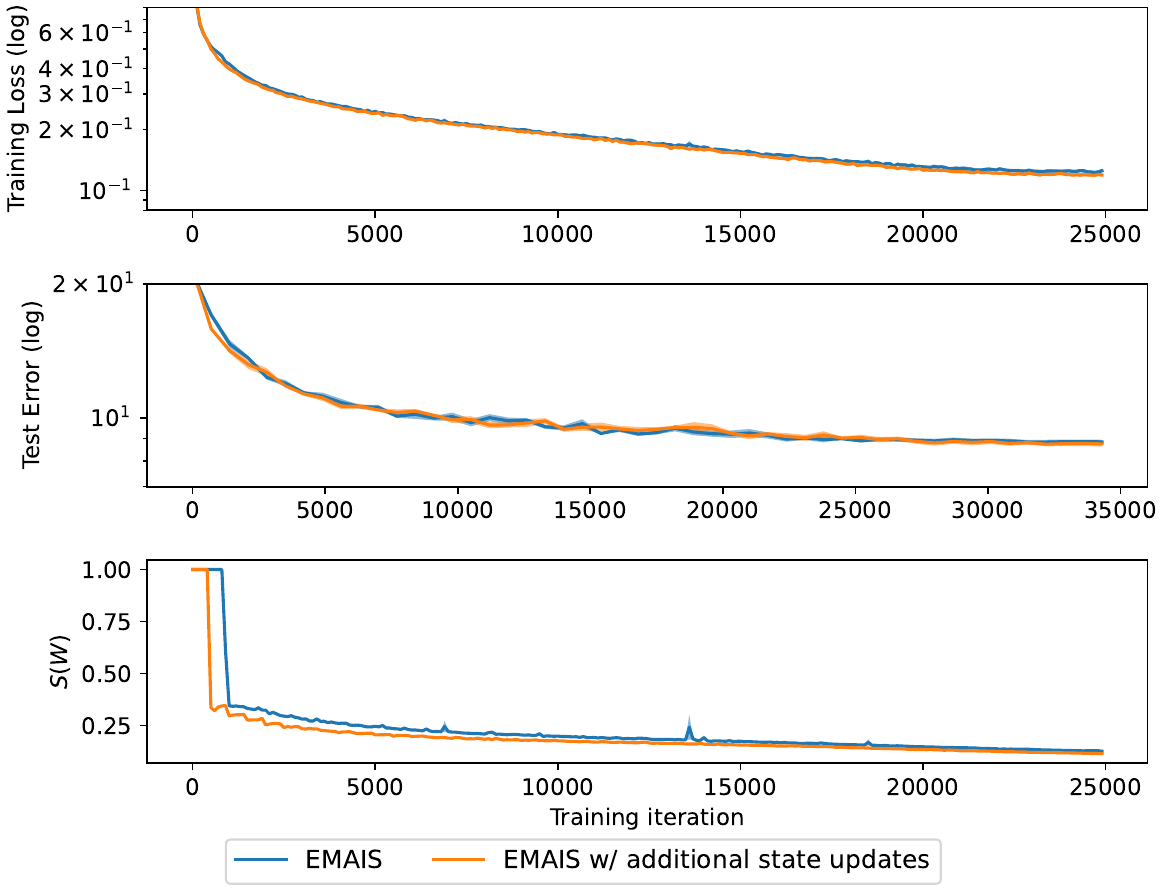}}
    \subfloat[CINIC-10]{\label{fig:additional_state_updates:cinic10}\includegraphics[width=0.5\linewidth,clip]{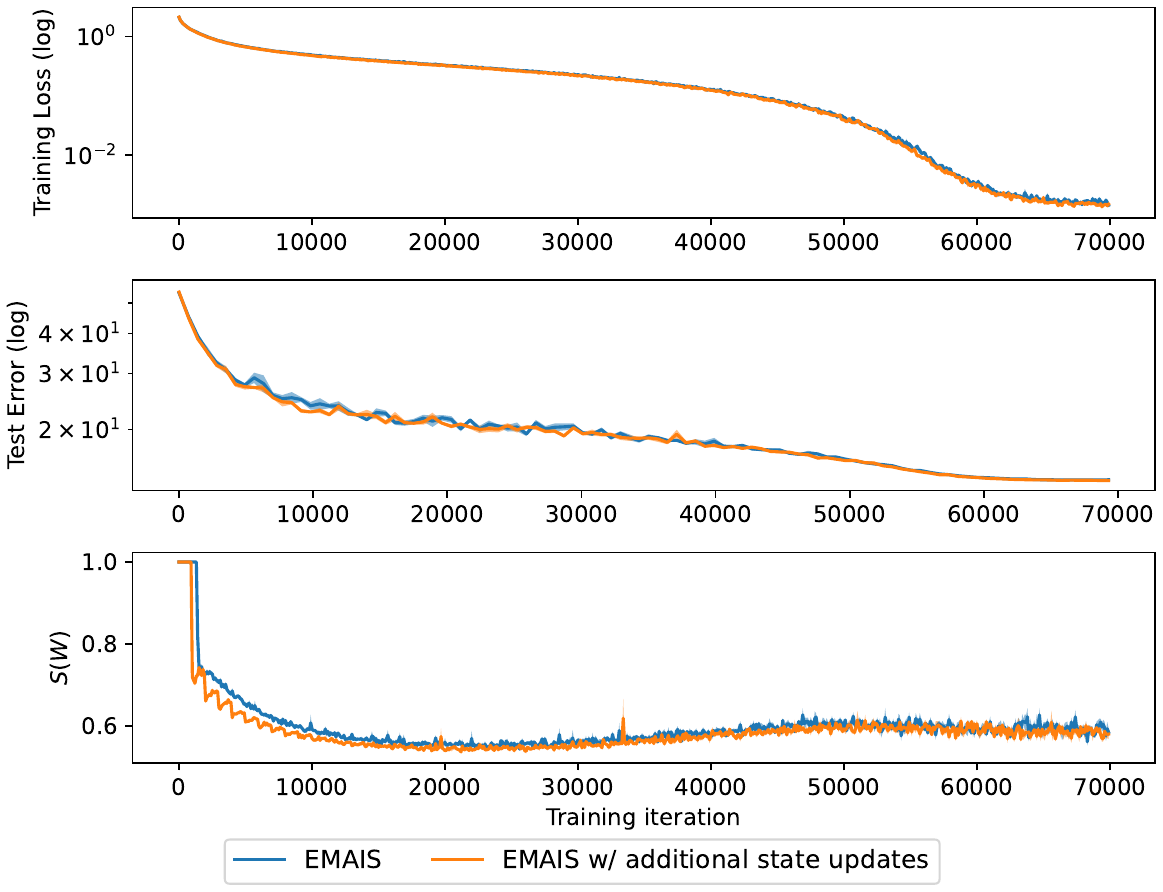}}
    \caption{%
        Performance of EMAIS with periodic additional updates of internal state variables on FMNIST and CINIC-10.%
    }
    \label{fig:additional_state_updates}
\end{figure}

\section{Conclusion}
We proposed a method for estimating the variance reduction of loss gradient estimators in DNN training using importance sampling, compared to uniform sampling. Since the proposed method requires only minibatches generated through importance sampling, it enables the online estimation of the variance reduction rate during DNN training.
By leveraging the proposed method, we also proposed an EMS to enable automatic learning rate adjustment. We developed an absolute metric to evaluate the efficiency of importance sampling and designed an algorithm to estimate importance weight using moving statistics.
Through numerical experiments, we demonstrated that the proposed method consistently achieved higher accuracy than uniform sampling, while maintaining comparable computational time. We also showed that the proposed method outperformed related importance-sampling methods, highlighting its effectiveness and efficiency.

As a potential extension, it would be worth exploring the application of our method to distributed learning.
Alain et al. \cite{alain2016variance} proposed a distributed importance sampling approach in which, at every training iteration, the gradient norm for each training sample is computed across workers, aggregated on a master node, and then used for importance sampling and model updates.
In contrast, our method estimates importance weights based on moving statistics of gradient norms, which may reduce the need for communication at every iteration. Furthermore, under a fixed data-sharding regime, each node may only need to communicate the scalar sum of its local weight vector~$W$ to reconstruct the global sampling distribution and perform importance sampling entirely on-node, potentially reducing communication costs even further.

\appendix

\section{Proofs and formula derivations} \label{apdx:derivation}
\subsection{Proof of \cref{prop:exp_equiv}} \label{apdx:exp_equiv}
For any function~$g(i)$ defined over~$i \in \llbracket M \rrbracket$, the following holds:
\begin{align}
    \mathbb{E}_{i \sim p_\mathrm{unif}(i)}\left[g(i)\right] & = \sum_{i=1}^M p_\mathrm{unif}(i) g(i) \label{eq:exp_equiv}                                  \\
                                                            & = \sum_{i=1}^M p_\mathrm{is}(i; W) \frac{p_\mathrm{unif}(i)}{p_\mathrm{is}(i;W)} g(i) \notag \\
                                                            & = \sum_{i=1}^M p_\mathrm{is}(i; W) r\left(i; W\right) g(i) \notag                            \\
                                                            & = \mathbb{E}_{i \sim p_\mathrm{is}(i;W)}\left[r(i;W) g(i)\right].  \notag
\end{align}
Therefore, by setting~$g(i)=\nabla_\theta\mathcal{L}(\theta; i)$, we obtain
\begin{align*}
    \mathbb{E}_{i \sim p_\mathrm{unif}(i)}\left[\nabla_\theta\mathcal{L}(\theta; i)\right] = \mathbb{E}_{i \sim p_\mathrm{is}(i;W)}\left[r(i;W) \nabla_\theta\mathcal{L}(\theta; i)\right],
\end{align*}
which indicates that $\mathbb{E}_\mathrm{unif}\left[\nabla_\theta\mathcal{L}(\theta)\right] = \mathbb{E}_{\mathrm{is}(W)}\left[\nabla_\theta\mathcal{L}(\theta)\right]$.

\subsection{Proof of \cref{prop:tr_cov_est}} \label{apdx:proof_prop_tr_cov_est}
From \cref{apdx:exp_equiv}, it holds that
\begin{align}
    \mu =  \mathbb{E}_{i \sim p_\mathrm{is}(i;W)}\left[r(i;W) \nabla_\theta \mathcal{L}(\theta;i)\right] = \mathbb{E}_{i \sim p_\mathrm{unif}(i)}\left[\nabla_\theta \mathcal{L}(\theta;i)\right]. \notag
\end{align}
We denote the expected value of the loss gradient with respect to the $k$-th parameter~$\theta_k$ by~$\mu_k$ as
\begin{align}
    \mu_k =  \mathbb{E}_{i \sim p_\mathrm{is}(i;W)}\left[r(i;W) \nabla_{\theta_k} \mathcal{L}(\theta;i)\right] = \mathbb{E}_{i \sim p_\mathrm{unif}(i)}\left[\nabla_{\theta_k} \mathcal{L}(\theta;i)\right] \notag
\end{align}
Then~$\tr\left(\mathbb{V}_{\mathrm{is}(W)}\left[\nabla_\theta \mathcal{L}(\theta)\right]\right)$ can be rewritten as
\begin{align}
     & \tr\left(\mathbb{V}_{\mathrm{is}(W)}\left[\nabla_\theta \mathcal{L}(\theta)\right]\right)                                                                                                                                                                  \notag \\
     & =\sum_{k=1}^K \mathbb{V}_{\mathrm{is}(W)}\left[\nabla_{\theta_k} \mathcal{L}(\theta)\right] \notag                                                                                                                                                                \\
     & = \sum_{k=1}^K \mathbb{E}_{i \sim p_\mathrm{is}(i;W)} \left[\left(r(i;W) \nabla_{\theta_k}\mathcal{L}(\theta;i) - \mu_{k}\right)^2\right] \notag                                                                                                                  \\
     & = \mathbb{E}_{i \sim p_\mathrm{is}(i;W)} \left[\sum_{k=1}^K \left(r(i;W) \nabla_{\theta_k}\mathcal{L}(\theta;i) - \mu_{k}\right)^2\right] \notag                                                                                                                  \\
     & = \mathbb{E}_{i \sim p_\mathrm{is}(i;W)} \left[\left\|r(i;W) \nabla_{\theta}\mathcal{L}(\theta;i) - \mu\right\|^2\right] \notag                                                                                                                                   \\
     & = \mathbb{E}_{i \sim p_\mathrm{is}(i;W)} \left[\left\|r(i;W) \nabla_{\theta}\mathcal{L}(\theta;i)\right\|^2\right] - 2\mu^\top \mathbb{E}_{i \sim p_\mathrm{is}(i;W)} \left[r(i;W)\nabla_{\theta}\mathcal{L}(\theta;i)\right]  + \left\|\mu\right\|^2 \notag      \\
     & = \mathbb{E}_{i \sim p_\mathrm{is}(i;W)} \left[\left\|r(i;W) \nabla_{\theta}\mathcal{L}(\theta;i)\right\|^2\right] - \left\|\mu\right\|^2. \notag
\end{align}
Note that this result, although derived through a different procedure, is consistent with the result of Alain et al. \cite{alain2016variance}.
Moreover,~$\tr\left(\mathbb{V}_{\mathrm{unif}}\left[\nabla_\theta \mathcal{L}(\theta)\right]\right)$ can be rewritten as
\begin{align}
    \tr\left(\mathbb{V}_\mathrm{unif}\left[\nabla_\theta \mathcal{L}(\theta)\right]\right)
     & = \sum_{k=1}^K \mathbb{V}_\mathrm{unif}\left[\nabla_{\theta_k} \mathcal{L}(\theta)\right] \notag                                                \\
     & = \sum_{k=1}^K \mathbb{E}_{i \sim p_\mathrm{unif}(i)} \left[\left(\nabla_{\theta_k}\mathcal{L}(\theta;i) - \mu_{k}\right)^2\right] \notag       \\
     & = \mathbb{E}_{i \sim p_\mathrm{unif}(i)} \left[\left\|\nabla_{\theta}\mathcal{L}(\theta;i) - \mu\right\|^2\right] \notag                        \\
     & = \mathbb{E}_{i \sim p_\mathrm{unif}(i)} \left[\left\|\nabla_{\theta}\mathcal{L}(\theta;i)\right\|^2\right] - \left\|\mu\right\|^2 \notag       \\
     & = \mathbb{E}_{i \sim p_\mathrm{is}(i;W)}\left[r(i;W) \left\|\nabla_\theta \mathcal{L}(\theta;i)\right\|^2\right] - \left\|\mu\right\|^2, \notag
\end{align}
where the final transformation follows from \cref{eq:exp_equiv} with~$g(i) = \left\|\nabla_{\theta}\mathcal{L}(\theta;i)\right\|^2$.

On the basis of \cref{eq:tr_V_is_ast},~$\tr\left(\mathbb{V}_{\mathrm{is}(W^\ast)}\left[\nabla_\theta \mathcal{L}(\theta)\right]\right)$ can be rewritten as
\begin{align}
    \tr\left(\mathbb{V}_{\mathrm{is}(W^\ast)}\left[\nabla_\theta \mathcal{L}(\theta)\right]\right)
     & = \left(\mathbb{E}_{i \sim p_\mathrm{unif}(i)}\left[\left\|\nabla_\theta \mathcal{L}(\theta;i)\right\|\right]\right)^2 - \left\|\mathbb{E}_{i \sim p_\mathrm{unif}(i)}\left[\nabla_\theta \mathcal{L}(\theta;i)\right]\right\|^2 \notag \\
     & = \left(\mathbb{E}_{i \sim p_\mathrm{is}(i;W)}\left[r(i;W) \left\|\nabla_\theta \mathcal{L}(\theta;i)\right\|\right]\right)^2 - \left\|\mu\right\|^2, \notag
\end{align}
where the transformation follows from \cref{eq:exp_equiv} with~$g(i) = \left\|\nabla_{\theta}\mathcal{L}(\theta;i)\right\|$.

\subsection{Proof of \cref{prop:ems}} \label{apdx:proof_prop_ems}
In training with a minibatch of size~$N'$ under uniform sampling, the following statistic is computed for the loss gradient estimation:
\begin{align}
    \nabla_\theta \bar{\mathcal{L}}_\mathrm{unif}^{N'} \left(\theta\right) := \frac{1}{N'} \sum_{k=1}^{N'} \nabla_\theta \mathcal{L}\left(\theta; i'_k\right) \text{ with } i'_k \sim p_\mathrm{unif}(i). \notag %\label{eq:loss_unif_n}
\end{align}
Assuming that~$i'_k$ are i.i.d., the mean and covariance matrix of~$\nabla_\theta \bar{\mathcal{L}}_\mathrm{unif}^{N'} \left(\theta\right)$ are given, respectively, by\footnote{See \cref{apdx:sample_mean} for the detailed derivation.}
\begin{align*}
    \mathbb{E}\left[\nabla_\theta \bar{\mathcal{L}}_\mathrm{unif}^{N'} \left(\theta\right)\right] & = \mathbb{E}_\mathrm{unif}\left[\nabla_\theta\mathcal{L}(\theta)\right],             \\
    \mathbb{V}\left[\nabla_\theta \bar{\mathcal{L}}_\mathrm{unif}^{N'} \left(\theta\right)\right] & = \frac{1}{N'}\mathbb{V}_\mathrm{unif}\left[\nabla_\theta\mathcal{L}(\theta)\right].
\end{align*}
Note that the expectation and variance on the left sides of the above equations are taken over the minibatch distribution.

Similarly, in training with a minibatch of size~$N$ under importance sampling with~$W$, the following statistic is computed for the loss gradient estimation:
\begin{align}
    \nabla_\theta \bar{\mathcal{L}}_{\mathrm{is}(W)}^N \left(\theta\right) := \frac{1}{N} \sum_{k=1}^N r\left(i_k;W\right) \nabla_\theta \mathcal{L}\left(\theta; i_k\right) \text{ with } i_k \sim p_\mathrm{is}(i;W). \notag
\end{align}
Under the i.i.d. assumption of~$i_k$, the mean and covariance matrix of $\nabla_\theta \bar{\mathcal{L}}_{\mathrm{is}(W)}^N \left(\theta\right)$ are given, respectively, by
\begin{align*}
    \mathbb{E}\left[\nabla_\theta \bar{\mathcal{L}}_{\mathrm{is}(W)}^N \left(\theta\right)\right] & = \mathbb{E}_{\mathrm{is}(W)}\left[\nabla_\theta\mathcal{L}(\theta)\right],            \\
    \mathbb{V}\left[\nabla_\theta \bar{\mathcal{L}}_{\mathrm{is}(W)}^N \left(\theta\right)\right] & = \frac{1}{N}\mathbb{V}_{\mathrm{is}(W)}\left[\nabla_\theta\mathcal{L}(\theta)\right].
\end{align*}
Then, from \cref{eq:grad_mean_eq}, it follows that
\begin{align*}
    \mathbb{E}\left[\nabla_\theta \bar{\mathcal{L}}_\mathrm{unif}^{N'} \left(\theta\right)\right] = \mathbb{E}\left[\nabla_\theta \bar{\mathcal{L}}_{\mathrm{is}(W)}^N \left(\theta\right)\right].
\end{align*}
Moreover, if~$N'$ is assumed to be equal to~$N_\mathrm{ems}$ in \cref{eq:n_ems}, it holds that
\begin{align*}
    \tr\left(\mathbb{V}\left[\nabla_\theta \bar{\mathcal{L}}_\mathrm{unif}^{N'} \left(\theta\right)\right]\right) & = \tr\left(\mathbb{V}\left[\nabla_\theta \bar{\mathcal{L}}_\mathrm{unif}^{N_\mathrm{ems}} \left(\theta\right)\right]\right)                                                                                                                                                      \\
                                                                                                                  & = \tr\left(\frac{\tr\left(\mathbb{V}_{\mathrm{is}(W)}\left[\nabla_\theta\mathcal{L}(\theta)\right]\right)}{\tr\left(\mathbb{V}_\mathrm{unif}\left[\nabla_\theta\mathcal{L}(\theta)\right]\right) N} \mathbb{V}_\mathrm{unif}\left[\nabla_\theta\mathcal{L}(\theta)\right]\right) \\
                                                                                                                  & = \tr\left(\frac{1}{N}\mathbb{V}_{\mathrm{is}(W)}\left[\nabla_\theta\mathcal{L}(\theta)\right]\right)
    = \tr\left(\mathbb{V}\left[\nabla_\theta \bar{\mathcal{L}}_{\mathrm{is}(W)}^N \left(\theta\right)\right]\right).
\end{align*}

\subsection{Mean and variance of sample mean} \label{apdx:sample_mean}
Consider a random variable~$x \in \mathbb{R}^n$ following a distribution~$p_x$, with mean~$\mu$ and covariance matrix~$\Sigma$.
Given~$L$ i.i.d. samples of~$x$, the sample mean of~$x$ is defined as
\begin{align}
    \bar{x} := \frac{1}{L}\sum_{i=1}^L x_i \quad \text{with} \quad x_i \sim p_x. \notag
\end{align}
The mean and covariance matrix of the sample mean~$\bar{x}$ are then given as
\begin{align}
    \mathbb{E}\left[\bar{x}\right] & = \mathbb{E}\left[\frac{1}{L}\sum_{i=1}^L x_i\right] = \frac{1}{L}\sum_{i=1}^L \mathbb{E}\left[x_i\right] = \frac{1}{L} L \mu = \mu, \notag                                                                                                         \\
    \mathbb{V}\left[\bar{x}\right] & = \mathbb{V}\left[\frac{1}{L}\sum_{i=1}^L x_i\right] = \frac{1}{L^2}\sum_{i=1}^L \mathbb{V}\left[x_i\right] + \frac{1}{L^2} \sum_{i=1}^L \sum_{j=1, j\neq i}^L \text{Cov}\left[x_i, x_j\right] = \frac{1}{L^2} L \Sigma = \frac{1}{L}\Sigma, \notag
\end{align}
where~$\text{Cov}$ denotes the covariance, and the assumption of i.i.d. samples indicates~$\text{Cov}\left[x_i, x_j\right] = 0$ for~$i \neq j$.

\section{Pseudo codes for variance estimation} \label{apdx:pseudo_code}
Assuming the use of Python~\cite{python3} and PyTorch~\cite{pytorch}, we present the pseudocode for variance estimation based on \cref{prop:tr_cov_est} in \cref{fig:pseudo_code}.
As shown in the pseudocode, by providing~$r(i;W)$ and the loss gradients for each data point in a minibatch generated in accordance with~$p_\mathrm{is}(i;W)$, the traces of the covariance matrices described in \cref{prop:tr_cov_est} can be efficiently estimated with only a few lines of code.
\begin{figure}[tbhp]
    \begin{lstlisting}[language=Python]
  def compute_tr_vars(r, per_sample_grad):
      """ Estimate the trace of variance-covariance matrices.
      Mini-batch is assumed to be generated following p_is.

      Args:
          r ([batch_size]): p_uniform / p_is
          per_sample_grad ([batch_size, n_params]): gradients for each sample
      """
      # Compute mu
      mu = (r[:, None] * per_sample_grad).mean(dim=0)

      # Variance of grad with current importance sampling
      tr_var_is = (r[:, None] * per_sample_grad).norm(p=2, dim=1).pow(2).mean()
      tr_var_is -= mu.norm(p=2).pow(2)

      # Variance of grad with uniform sampling
      tr_var_uniform = (r * per_sample_grad.norm(p=2, dim=1).pow(2)).mean()
      tr_var_uniform -= mu.norm(p=2).pow(2)

      # Variance of grad with optimal importance sampling (lower-bound)
      sample_grad_norm = per_sample_grad.norm(p=2, dim=1)
      tr_var_optimal_is = (r * sample_grad_norm).mean().pow(2)
      tr_var_optimal_is -= mu.norm(p=2).pow(2)

      return tr_var_is, tr_var_uniform, tr_var_optimal_is
  \end{lstlisting}
    \caption{Pseudo code for variance estimation based on \cref{prop:tr_cov_est}}
    \label{fig:pseudo_code}
\end{figure}

\section{Additional experimental results}

\subsection{Efficiency score~$\mathcal{S}(W)$ and estimated trace of gradient variances} \label{apdx:score_and_trvars}
\Cref{fig:score_and_trvars} shows the efficiency score~$\mathcal{S}(W)$ and the corresponding estimated trace of gradient variances used in its computation, for CINIC-10 and FMNIST training with EMAIS using Linear-$\tau$.
Note that the trace of the gradient variance is approximated using logit gradients, as discussed in \cref{sec:approx_by_logit}.
\begin{figure}[tbhp]
    \centering
    \subfloat[CINIC-10]{\includegraphics[width=0.5\linewidth,clip]{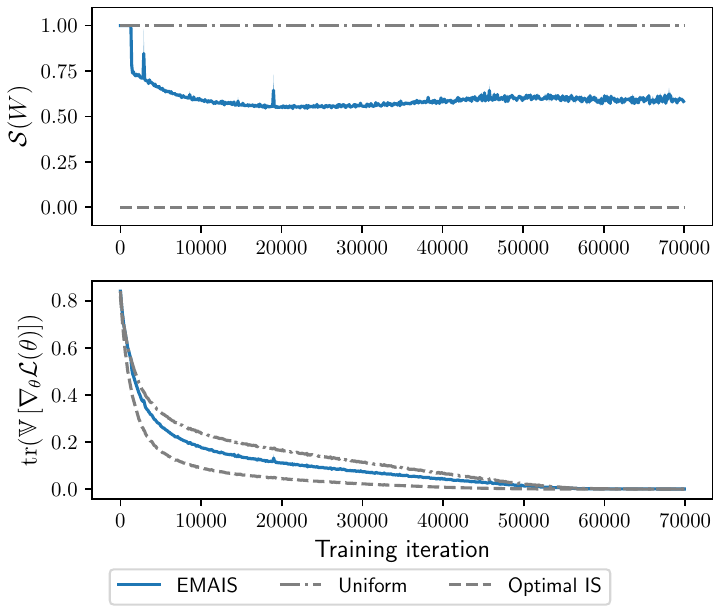}}
    \subfloat[FMNIST]{\includegraphics[width=0.5\linewidth,clip]{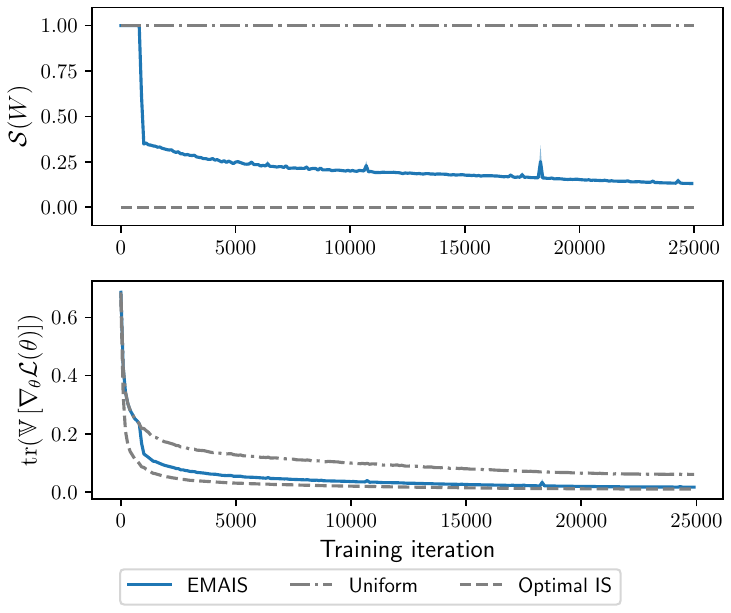}}
    \caption{Efficiency score~$\mathcal{S}(W)$ and corresponding estimated trace of gradient variances for CINIC-10 and FMNIST.}
    \label{fig:score_and_trvars}
\end{figure}

\subsection{Analysis of Linear-$\tau$ with varying total iterations} \label{apdx:analysis_linear_tau}
We conducted training for CINIC-10 using Linear-$\tau$ while varying the total number of training iterations as 35000, 70000, 105000, or 140000. The evaluated~$\mathcal{S}(W)$ during training is presented in \cref{fig:cinic10_35000-140000_IS_score}.
\begin{figure}[tbhp]
    \centering
    \includegraphics[width=0.7\linewidth,clip]{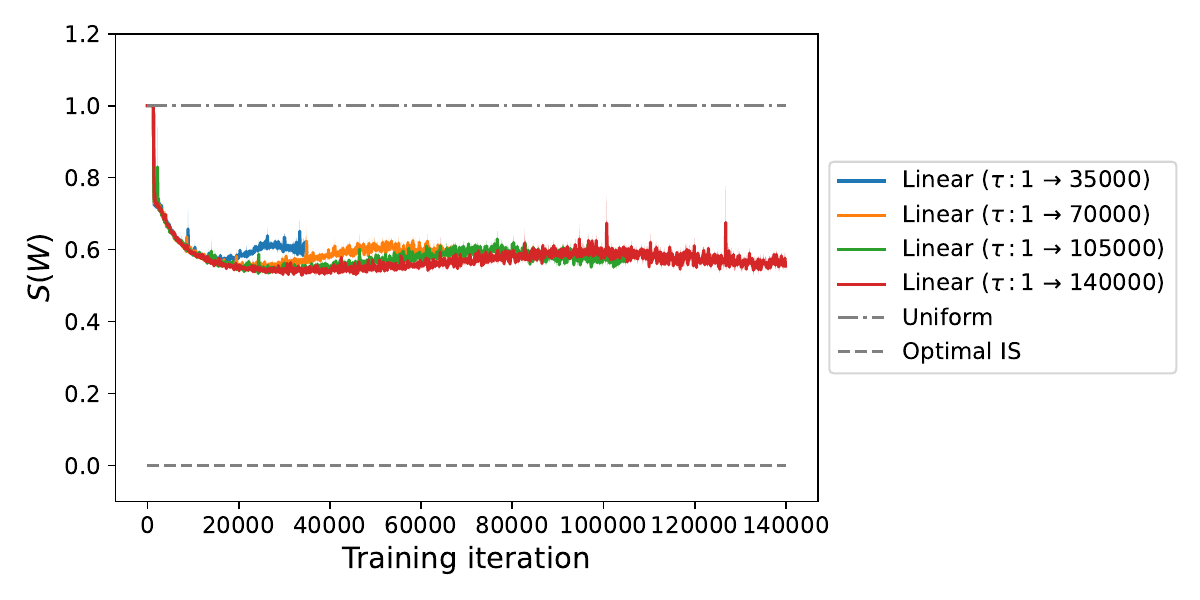}
    \caption{Transition of~$\mathcal{S}(W)$ during training with importance sampling under Linear-$\tau$ strategy for varying total training iterations on CINIC-10.}
    \label{fig:cinic10_35000-140000_IS_score}
\end{figure}

\subsection{Training time} \label{apdx:training_time}
The training times for DNN models on FMNIST, CINIC-10, and ChestX-ray14 are summarized in \cref{tab:res_time_fmnist}, \cref{tab:res_time_cinic10}, and \cref{tab:res_time_nih}, respectively.
Each table reports the mean training time in seconds, with the numbers in parentheses representing the standard deviations.

\begin{table}[tbhp]
    \centering
    % \footnotesize
    \tabcolsep = 5pt
    \caption{Training time (s) for FMNIST}
    \label{tab:res_time_fmnist}
    \scalebox{0.67}{
        \begin{tabular}{r|rrrrrrrr}
            \toprule
            Total-iters & \multicolumn{1}{c}{SGD-scan} & \multicolumn{1}{c}{SGD-unif} & \multicolumn{1}{c}{RAIS} & \multicolumn{1}{c}{Presampling-IS} & \multicolumn{1}{c}{Confidence} & \multicolumn{1}{c}{ConfVar} & \multicolumn{1}{c}{Self-paced} & \multicolumn{1}{c}{EMAIS} \\
            \midrule
            6250        & $53.1 (\pm 0.2)$             & $53.5 (\pm 0.3)$             & $68.6 (\pm 0.7)$         & $321.8 (\pm 1.6)$                  & $40.4 (\pm 0.9)$               & $58.1 (\pm 1.5)$            & $53.4 (\pm 0.4)$               & $48.0 (\pm 1.2)$          \\
            12500       & $108.0 (\pm 0.8)$            & $107.8 (\pm 0.9)$            & $135.3 (\pm 6.4)$        & $656.0 (\pm 1.5)$                  & $76.6 (\pm 1.2)$               & $114.3 (\pm 3.8)$           & $108.2 (\pm 0.4)$              & $90.1 (\pm 2.3)$          \\
            25000       & $216.2 (\pm 2.4)$            & $216.9 (\pm 2.2)$            & $262.3 (\pm 14.5)$       & $1316.2 (\pm 12.2)$                & $151.2 (\pm 2.4)$              & $223.7 (\pm 4.2)$           & $219.3 (\pm 1.1)$              & $176.9 (\pm 3.6)$         \\
            \bottomrule
        \end{tabular}
    }
\end{table}

\begin{table}[tbhp]
    \centering
    % \footnotesize
    \tabcolsep = 5pt
    \caption{Training time (s) for CINIC-10}
    \label{tab:res_time_cinic10}
    \scalebox{0.58}{
        \begin{tabular}{lr|rrrrrrrr}
            \toprule
            WD   & Total-iters & \multicolumn{1}{c}{SGD-scan} & \multicolumn{1}{c}{SGD-unif} & \multicolumn{1}{c}{RAIS} & \multicolumn{1}{c}{Presampling-IS} & \multicolumn{1}{c}{Confidence} & \multicolumn{1}{c}{ConfVar} & \multicolumn{1}{c}{Self-paced} & \multicolumn{1}{c}{EMAIS} \\
            \midrule
            1e-4 & 17500       & $1169.3 (\pm 76.4)$          & $1111.3 (\pm 13.7)$          & $1188.2 (\pm 65.4)$      & $3496.6 (\pm 139.2)$               & $1116.2 (\pm 30.7)$            & $1167.5 (\pm 29.6)$         & $1137.5 (\pm 33.2)$            & $1153.6 (\pm 47.7)$       \\
                 & 35000       & $2329.3 (\pm 167.3)$         & $2322.7 (\pm 151.5)$         & $2327.2 (\pm 50.1)$      & $7140.6 (\pm 397.3)$               & $2250.7 (\pm 68.3)$            & $2342.5 (\pm 66.6)$         & $2272.0 (\pm 39.9)$            & $2280.6 (\pm 87.5)$       \\
                 & 70000       & $4693.2 (\pm 168.4)$         & $4498.9 (\pm 27.3)$          & $4552.2 (\pm 38.5)$      & $13596.8 (\pm 287.9)$              & $4416.5 (\pm 115.1)$           & $4591.8 (\pm 49.6)$         & $4601.4 (\pm 91.5)$            & $4546.3 (\pm 88.0)$       \\
                 & 140000      & $9202.0 (\pm 303.0)$         & $9068.9 (\pm 165.7)$         & $9071.8 (\pm 94.4)$      & $28065.0 (\pm 802.4)$              & $8838.7 (\pm 275.4)$           & $9318.1 (\pm 156.0)$        & $9123.9 (\pm 228.8)$           & $9065.2 (\pm 205.7)$      \\
            % \midrule
            \addlinespace
            1e-5 & 17500       & $1164.0 (\pm 54.2)$          & $1174.8 (\pm 63.4)$          & $1183.8 (\pm 61.2)$      & $3443.2 (\pm 120.3)$               & $1114.6 (\pm 39.3)$            & $1185.3 (\pm 27.3)$         & $1112.9 (\pm 21.3)$            & $1237.2 (\pm 29.2)$       \\
                 & 35000       & $2250.6 (\pm 50.7)$          & $2415.6 (\pm 163.2)$         & $2307.8 (\pm 43.3)$      & $6885.3 (\pm 211.3)$               & $2219.1 (\pm 63.6)$            & $2364.7 (\pm 58.1)$         & $2254.2 (\pm 36.4)$            & $2473.4 (\pm 37.8)$       \\
                 & 70000       & $4697.9 (\pm 170.7)$         & $4605.3 (\pm 304.8)$         & $4631.1 (\pm 79.8)$      & $13595.9 (\pm 285.0)$              & $4405.0 (\pm 121.9)$           & $4724.7 (\pm 151.7)$        & $4617.2 (\pm 52.1)$            & $4620.0 (\pm 223.4)$      \\
                 & 140000      & $9124.9 (\pm 361.5)$         & $9739.0 (\pm 712.0)$         & $9328.6 (\pm 195.0)$     & $27772.9 (\pm 306.7)$              & $8867.0 (\pm 338.4)$           & $9433.2 (\pm 268.2)$        & $9087.3 (\pm 211.2)$           & $9195.6 (\pm 357.6)$      \\
            \bottomrule
        \end{tabular}
    }
\end{table}

\begin{table}[tbhp]
    \centering
    % \footnotesize
    \tabcolsep = 5pt
    \caption{Training time (s) for ChestX-ray14}
    \label{tab:res_time_nih}
    \scalebox{0.7}{
        \begin{tabular}{r|rrrrrr}
            \toprule
            Total-iters & \multicolumn{1}{c}{SGD-scan} & \multicolumn{1}{c}{SGD-unif} & \multicolumn{1}{c}{RAIS} & \multicolumn{1}{c}{Presampling-IS} & \multicolumn{1}{c}{Self-paced} & \multicolumn{1}{c}{EMAIS} \\
            \midrule
            17500       & $7768.5 (\pm 14.2)$          & $7758.2 (\pm 2.0)$           & $7866.1 (\pm 43.1)$      & $23824.4 (\pm 25.7)$               & $7759.7 (\pm 3.1)$             & $7773.8 (\pm 3.1)$        \\
            35000       & $15580.2 (\pm 5.1)$          & $15587.5 (\pm 2.4)$          & $15724.8 (\pm 62.4)$     & $47792.9 (\pm 42.5)$               & $15583.7 (\pm 6.7)$            & $15608.3 (\pm 7.2)$       \\
            70000       & $31223.1 (\pm 7.7)$          & $31223.7 (\pm 13.7)$         & $31487.9 (\pm 115.0)$    & $95637.2 (\pm 92.0)$               & $31222.2 (\pm 13.2)$           & $31268.2 (\pm 10.2)$      \\
            \bottomrule
        \end{tabular}
    }
\end{table}

\bibliographystyle{plainnat}
\bibliography{references}

\end{document}